\documentclass[11pt]{article}

\usepackage[final]{acl}

\usepackage{times}
\usepackage{latexsym}
\usepackage[T1]{fontenc}

\usepackage[utf8]{inputenc}

\usepackage{microtype}
\usepackage{listings}
\usepackage{inconsolata}
\usepackage{CJKutf8}
\usepackage[encapsulated]{CJK}
\usepackage{graphicx}
\usepackage{subcaption}
\usepackage{amsmath}
\usepackage{amssymb}
\usepackage{booktabs} 
\usepackage{multirow} 
\usepackage{diagbox} 
\usepackage{xcolor}
\usepackage{tabularx}
\usepackage{amsthm}
\usepackage{capt-of}

\newtheorem{corollary}{Corollary}
\usepackage[most]{tcolorbox}
\definecolor{myrock}{HTML}{6202ff}    
\definecolor{mypaper}{HTML}{ff5722} 
\definecolor{myscissors}{HTML}{b9f3e9}    
\definecolor{mygreen}{HTML}{007f00}

\definecolor{string1}{HTML}{9c0293}   
\definecolor{string2}{HTML}{ddeefd} 
\definecolor{string3}{HTML}{ffcdd2}   

\definecolor{down}{HTML}{228B22} 
\definecolor{up}{HTML}{B22222}   
\usepackage[table]{xcolor}
\usepackage[most]{tcolorbox} 
\usepackage{titletoc}
\usepackage[title,toc]{appendix}

\usepackage{amsthm}

\newtheorem{proposition}{Proposition}
\theoremstyle{definition}
\newtheorem{assumption}{Assumption}

\tcbset{enhanced} 

\newtcolorbox{exercisebox}[1][]{ 
  colback=gray!5,            
  colframe=black!75,         
  coltitle=white,            
  fonttitle=\bfseries,       
  title=#1,                  
  boxrule=1pt,               
  arc=3pt,                   
  left=10pt,                 
  top=8pt,                   
  bottom=8pt,                
  before skip=10pt,          
  after skip=10pt,           
}

\usepackage{multicol}
\usepackage{newfloat}
\usepackage{listings}
\DeclareCaptionStyle{ruled}{labelfont=normalfont,labelsep=colon,strut=off} 
\title{Mutual Debiasing via Dual‑Seed Comparison for Probabilistic Sampling in Large Language Models}

\author{
\textbf{Zihao Guo}$^{1}$,
\textbf{Hongtao Lv}$^{1}$\thanks{Corresponding authors.},
\textbf{Chaoli Zhang}$^{2}$,
\textbf{Laiguo Yin}$^{1}$,
\textbf{Lei Liu}$^{1}$,
\textbf{Yonghui Xu}$^{1}$,
\textbf{Lizhen Cui}$^{1}$\footnotemark[1] \\
$^{1}$ School of Software, Shandong University, China \\
$^{2}$ School of Computer Science and Technology, Zhejiang Normal University, China\\
\texttt{\{zihaog, lgyin\}@mail.sdu.edu.cn}, \\
\texttt{\{lht, l.liu, xuyonghui, clz\}@sdu.edu.cn},\quad
\texttt{chaolizcl@zjnu.edu.cn}
}

\begin{document}
\maketitle
\begin{abstract}
Although Large Language Models (LLMs) demonstrate remarkable capabilities in reasoning and decision-making, high-fidelity probabilistic sampling remains a persistent challenge. When generating random variables, LLMs consistently exhibit systematic biases that warp the target probability distributions. Current approaches often rely on a single, self-generated seed, which inherits model-specific biases. To overcome this vulnerability, we introduce Dual-Seed Comparison (DSC), a transparent, tool-free protocol that utilizes two independent LLM-generated seeds to neutralize bias. DSC compares the character-level ordinal values of the two seeds to construct a bit sequence, converts and normalizes this sequence into a pseudo-uniform variate, and then maps the variate to the target distribution through the inverse cumulative distribution function (CDF). Empirical results show that DSC substantially outperforms existing methods across 96\% of evaluated settings. Beyond direct sampling, task-adapted variants based on the DSC comparison operator improve distributional control in MCQ generation and attribute-constrained text-to-image prompting.

\end{abstract}

\section{Introduction}

The expanding role of Large Language Models (LLMs) in agent-based simulations \cite{park2023generative} and synthetic data generation \cite{li-etal-2023-synthetic} has elevated rigorous probability sampling from a theoretical benchmark to a strict operational requirement \cite{zhao2026baddice}. Empirical evidence confirms that LLMs systematically fail at basic stochastic tasks \cite{randnumber,playdice,randomisrandom,guo-etal-2025-illusion}, and that their random number generation fails standard statistical test suites \cite{karanjai2025evaluating}. Crucially, this deficit is one of execution rather than knowledge: models can accurately describe target distributions \cite{DistributionalAlignmentof} and derive optimal mixed strategies \cite{effectof}, yet fail to act on them stochastically. This persistent ``cognition--behavior gap'' \cite{guo-etal-2025-illusion} across model scales highlights a fundamental structural limitation in autoregressive generation.

Given this structural limitation, existing approaches have primarily sought to bypass rather than resolve it. The dominant strategy externalizes the stochastic process entirely, either by prompting LLMs to invoke external libraries like \texttt{numpy.random} or by outsourcing the simulation to code-execution environments \cite{silva2024large}. This mirrors a recurring pattern in LLM reasoning research: when models struggle with arithmetic, a natural response is to offload computation to external tools \cite{10.5555/3618408.3618843, 10.5555/3666122.3669119}. However, just as recent literature insists that mathematical computation is a core cognitive skill rather than a peripheral task to be offloaded \cite{Dawid_2024, lewkowycz2022}, we argue that stochastic sampling demands similar scrutiny. If an LLM must call an external pseudo-random number generator (PRNG) script to draw a basic uniform variate, it has merely memorized the vocabulary of probability. The underlying functional competence remains absent \cite{MAHOWALD2024517}, leaving the fundamental generative limitation masked rather than resolved.

\begin{figure*}[t]
\centering
\includegraphics[width=\textwidth]{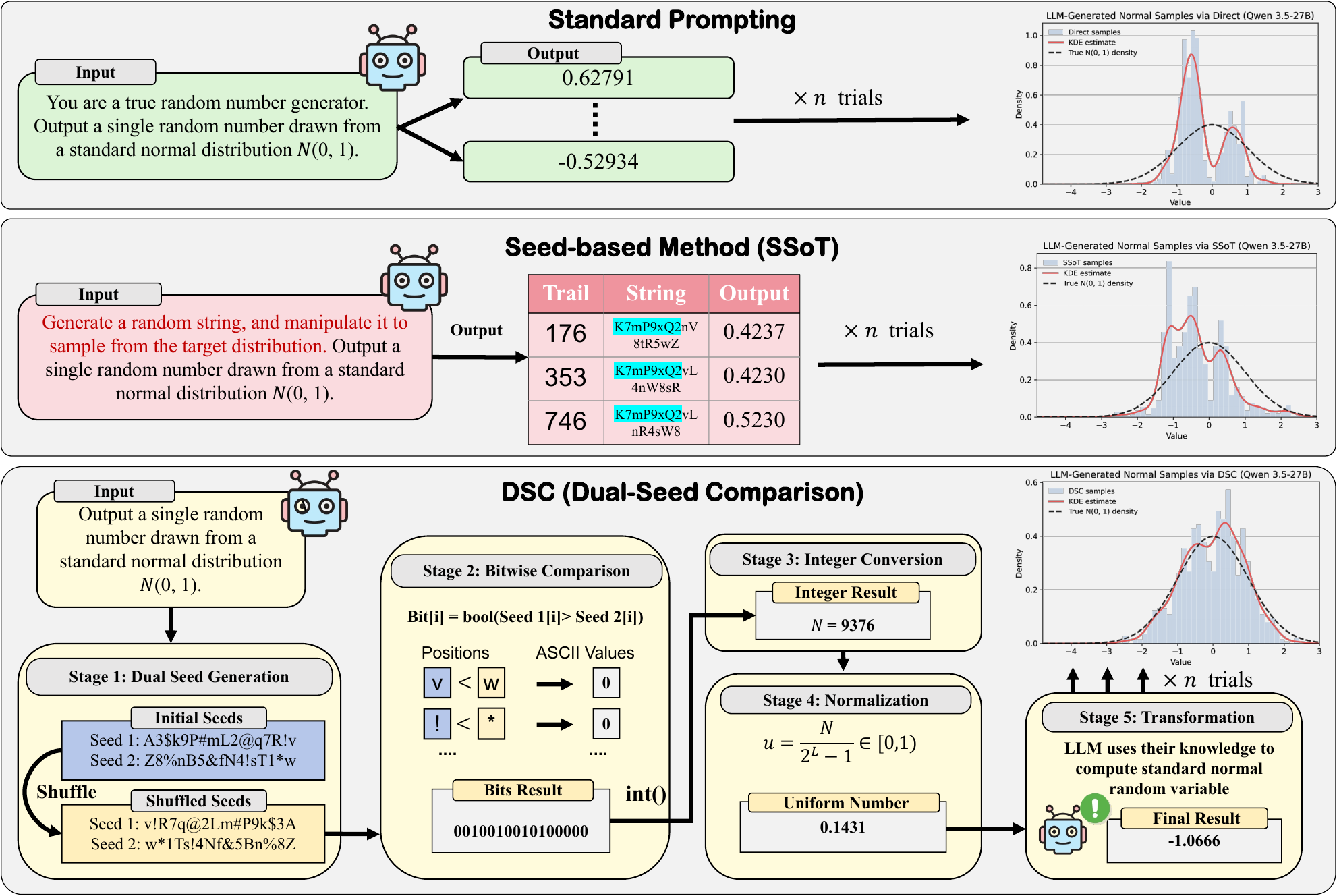}
\caption{Overview of LLM-based continuous sampling methodologies. (Top) Standard direct prompting fails to capture the continuous density, resulting in severe structural deviations from the target $\mathcal{N}(0,1)$. (Middle) Single-seed approaches like SSoT propagate the inherent character-level biases of the LLM, leading to skewed sampling distributions. (Bottom) The proposed Dual-Seed Comparison (DSC) pipeline. See Section~\ref{sec:method} for details. Graphs show the resulting empirical distributions, obtained using Qwen3.5-27B.}
\label{fig:main}
\end{figure*}

To bypass this dependence on external environments, a more intrinsic line of work seeks to extract randomness directly from outputs produced by the LLM itself. Specifically, String Seed of Thought \cite[SSoT;][]{misaki2025stringseedthoughtprompting} instructs the LLM to generate a random string as an entropy source and autonomously extract a categorical sample. While SSoT alleviates the above problem to some extent, it has two notable limitations. First, it relies entirely on the initial string being an unbiased character source. However, as illustrated in Figure~\ref{fig:main}, this assumption does not hold in practice. When testing SSoT on Qwen3.5-27B for sampling from an exponential distribution, the LLM repeatedly generates strings sharing the same prefix, such as ``K7mP9xQ2'' across multiple trials, even though these strings are expected to serve as independent seeds. Because SSoT maps these strings directly into samples, the repeated textual pattern leads to correlated and biased outputs, which in turn distort the empirical sample distribution away from the target distribution. We empirically expose this character-level bias in Section~\ref{sec:llmsfail}.  Second, its autonomous extraction logic functions as a black box, rendering the seed-to-sample mapping fundamentally opaque.

We bridge these gaps by introducing \textbf{Dual-Seed Comparison (DSC)}, a transparent protocol for probabilistic sampling. Instead of treating a single LLM-generated string as an unbiased character source, DSC prompts the model to generate two independent strings and compares their characters position by position. As shown in the bottom panel of Figure~\ref{fig:main}, as long as the two generated strings provide independent seeds, the ordinal comparison at each position yields a randomized binary outcome driven by their relative ordering, rather than by the absolute character preferences of the model. The resulting comparison-bit sequence is normalized into a pseudo-uniform variate $u \in [0,1)$, which is then mapped to the target continuous distribution through the inverse cumulative distribution function (CDF). This end-to-end process operates in a single inference call and produces a fully auditable trail of intermediate steps.

Our contributions are as follows:
\begin{itemize}
\item We empirically demonstrate that LLM-generated random strings exhibit systematic character-level biases, establishing that single-seed entropy sources are fundamentally unreliable for probabilistic sampling.
\item We propose DSC, a prompt-based method enabling LLMs to sample from probability distributions without external tools, via a transparent, multi-stage arithmetic pipeline.
\item Evaluated across five models and five distributions, DSC achieves the lowest KS statistic in 24 of 25 conditions. Beyond distribution sampling, DSC improves distribution fidelity in MCQ answer-position control and attribute-constrained prompt generation.
\end{itemize}

\section{Related Work}
\label{sec:back}

The inability of LLMs to produce high-quality random outputs has been documented across a widening range of tasks. Early controlled evaluations revealed systematic deviations from target distributions in uniform integer generation \cite{randnumber}, dice rolling \cite{playdice}, and coin flipping \cite{randomisrandom}—with GPT-4, for instance, sacrificing independence to maintain the appearance of uniformity \cite{playdice}. Subsequent work broadened the scope: \citet{guo-etal-2025-illusion} evaluated 20 LLMs in Rock-Paper-Scissors and found systematic deviations from the uniform mixed-strategy equilibrium. Their ablations showed that neither model scaling nor temperature tuning eliminated the broader randomness deficit. \citet{karanjai2025evaluating} applied the full NIST randomness test suite and reported that the best-performing LLM passed fewer than one-third of the tests, compared to over 87\% for Python pseudo-random generators.

These failures trace to structural properties of the autoregressive generation process. LLM output probabilities in numeric contexts are poorly calibrated, with biases driven by word identity, presentation order, and corpus frequency \cite{CalibratedinNumericContexts, pezeshkpour2023largelanguagemodelssensitivity, premiseOrderMatters}. The deficit, moreover, is one of execution rather than knowledge: models can accurately describe a target distribution yet fail to sample from it \cite{DistributionalAlignmentof}, and game-playing agents can derive the correct mixed strategy, but cannot act on it stochastically  \cite{effectof}. 

Recent studies have explored different mitigations for LLM sampling failures. \citet{silva2024large} improve mixed-strategy gameplay by granting models code-execution access, allowing randomization and calculation to be delegated to executable code. In a different direction, \citet{EnoughCoinFlips} use biased coin-flip sequences as in-context evidence to test whether LLM predictive probabilities update in a Bayesian manner; they find that models often have miscalibrated priors but can broadly follow Bayesian posterior updates after sufficient demonstrations. These results clarify what LLMs can recover through tool use or in-context evidence, but they do not directly provide an intrinsic sampling mechanism. Taking a more intrinsic approach, String Seed of Thought \cite[SSoT;][]{misaki2025stringseedthoughtprompting} prompts the LLM to first generate a random string to accumulate sufficient entropy, and then extract this randomness by manipulating the string to derive a final categorical sample, thereby preserving output diversity while adhering to specific constraints. SSoT provides theoretical guarantees that the total variation distance decreases exponentially with string length and, empirically, enables reasoning-capable models to approach PRNG-level fidelity on discrete benchmarks. Nevertheless, the intermediate steps of this string-to-sample mapping remain opaque and untrustworthy. 

A pervasive limitation in the current literature is its near-exclusive focus on discrete tasks, such as dice rolls, coin flips, and rock-paper-scissors. Recent diagnostic work by \citet{zhao2026baddice} expanded this evaluation to continuous domains, auditing 11 models across 15 diverse distributions. Their findings highlight a severe deficit: under independent stateless requests, 10 of the 11 models failed across every distribution tested. However, no existing methodology offers a constructive solution for this critical setting. Our work bridges this gap by introducing DSC. DSC converts two LLM-generated character strings into a pseudo-uniform variate on $[0,\, 1)$, and then maps it to the target continuous distribution through the quantile function. 

\section{Preliminaries}
\label{sec:prelim}

\subsection{Statistical Tests for Distributional Fidelity}

We evaluate the distributional fidelity of LLMs by comparing the empirical distribution of their generated samples with the target distribution $F_0$ (e.g., Uniform $\mathcal{U}(0,1)$). For a set of $n$ samples $\{x_1,\ldots,x_n\}$ produced by independent stateless queries to a model $\mathcal{M}$, we define the empirical CDF as
\begin{equation}
    F_n(x)=\frac{1}{n}\sum_{i=1}^{n}\mathbf{1}[x_i\le x].
\end{equation}

We then use two goodness-of-fit measures to quantify the discrepancy between $F_n$ and $F_0$.

\paragraph{Kolmogorov--Smirnov (KS) test.}
The KS statistic measures the supremum deviation from the target CDF:
\begin{equation}
  D_n = \sup_{x} \bigl|F_n(x) - F_0(x)\bigr|.
\end{equation}
A smaller $D_n$ indicates closer agreement with the target distribution $F_0$, with $D_n=0$ corresponding to an exact match. We report $D_n$ directly as our primary fidelity metric, rather than converting it into a $p$-value. This is because our goal is to quantify and compare the magnitude of distributional deviation across methods. In contrast, the KS $p$-value answers a binary hypothesis-testing question and becomes increasingly sensitive as $n$ grows, so even small deviations can lead to rejection without clearly indicating which method is closer to the target.

\paragraph{Wasserstein distance.}
The first-order Wasserstein distance (earth mover's distance) between
the empirical distribution and $F_0$ is
\begin{equation}
  W_1 = \int_{-\infty}^{\infty} \bigl|F_n(x) - F_0(x)\bigr|\,
  \mathrm{d}x.
\end{equation}
While $D_n$ evaluates fidelity based on the maximum pointwise divergence, $W_1$ integrates discrepancies across the entire distributional support. $W_1$ exhibits heightened sensitivity to systematic, low-magnitude biases that $D_n$ inherently overlooks. As with $D_n$, a lower $W_1$ reflects closer adherence to $F_0$, with $W_1 = 0$ strictly signifying an exact match between the empirical and target distributions.

\paragraph{Chi-square ($\chi^2$) test.}
For discrete sampling tasks, we additionally use the chi-square goodness-of-fit test to measure whether the empirical category frequencies match the target probabilities. 
Let $\mathcal{Y}=\{1,\ldots,J\}$ denote the discrete outcome space, with target probabilities $\pi_1,\ldots,\pi_J$. Given $n$ generated samples, let $O_j$ be the observed count of category $j$ and $E_j = n\pi_j$ be its expected count under the target distribution. The chi-square statistic is defined as
\begin{equation}
  \chi^2 = \sum_{j=1}^{J} \frac{(O_j - E_j)^2}{E_j}.
\end{equation}
A smaller $\chi^2$ value indicates that the empirical frequencies are closer to the prescribed target probabilities. Unlike KS and $W_1$, which are designed for continuous distributions, the chi-square statistic directly evaluates categorical frequency mismatch and is therefore used for discrete sampling evaluations.

\subsection{LLM-Generated Seeds Are Not Random}
\label{sec:llmsfail}

Seed-based sampling methods, such as SSoT, rely on the assumption that a single LLM-generated string serves as a reliable randomness source. We test this assumption directly. We prompt each model to generate 16-character random strings drawn from the 95 printable ASCII characters (codes 32-126), collecting 200 strings in each of 5 repeats (yielding 16{,}000 character-level samples) per model. If the model can select characters uniformly and independently, each character should appear with frequency $1/95$.

\begin{figure}[t]
\centering
\includegraphics[width=\columnwidth]{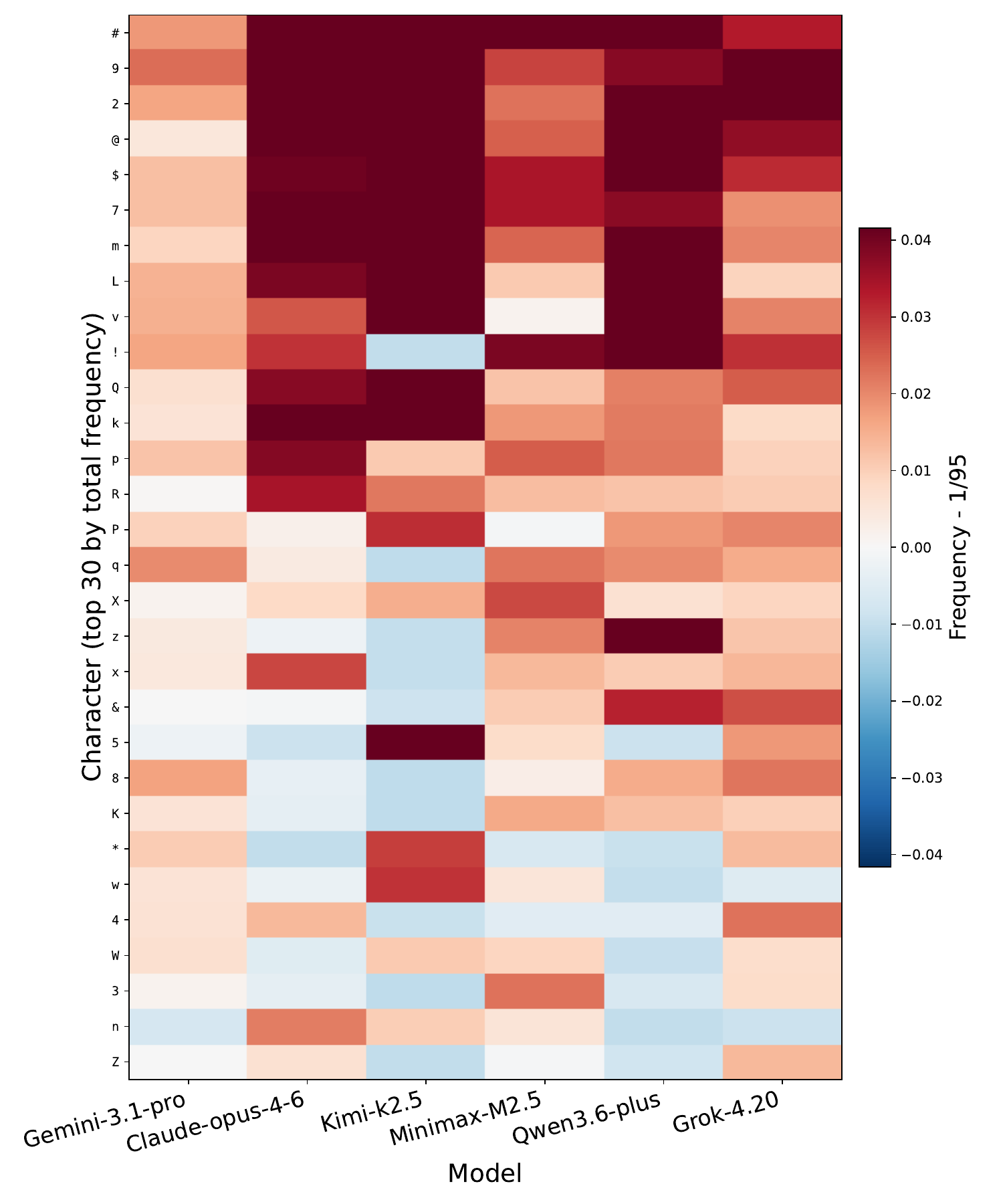}
\caption{
Character-frequency deviation from the printable-ASCII alphabet characters across 6 LLMs. Rows show the 30 characters with the highest total frequency across all models. Each cell reports the empirical frequency for a model minus $1/95 \approx .0105$; red indicates overrepresentation and blue indicates underrepresentation.
}
\label{fig:char_heatmap}
\end{figure}

Figure~\ref{fig:char_heatmap} visualizes the per-character frequency bias across models. Each cell represents the deviation from the ideal uniform baseline of $1/95 \approx 0.0105$: positive values (red) indicate overrepresentation, zero (white) denotes ideal frequency, and negative values (blue) signify underuse. Under ideal uniform sampling, the heatmap would be entirely white; instead, all models display pronounced column-wise biases. Certain characters, such as \texttt{\#}, \texttt{@}, digits, and some letters, are heavily favored, whereas others are nearly absent.

Table~\ref{tab:char_level} quantifies these observations. 
Although several models use a broad portion of the printable ASCII space, the KS test rejects character-level uniformity for all models, with mean KS statistics ranging from $.082$ to $.197$ and $p < .001$ across all tested models.\footnote{For each generated character $c$, we compute $u_c=(\mathrm{ord}(c)-32)/95 \in [0,1)$ and apply a one-sample KS test against $\mathcal{U}(0,1)$. We reject character-level uniformity when the resulting $p$-value is below the significance threshold $\alpha=0.05$.} 
Vocabulary diversity also varies substantially: Gemini 3.1 Pro uses 94 of 95 printable characters, whereas Kimi-K2.5 collapses to a restricted subset of only 29. 
These results show that LLM-generated strings are not reliable uniform seed sources: their character distributions deviate systematically from the uniform ideal, and any downstream method that directly treats these strings as random seeds inherits this bias.

\begin{table}[t]
\resizebox{\linewidth}{!}{
\centering
\small
\begin{tabular}{lccc}
\toprule
\textbf{Model} & \textbf{Unique Chars} & \textbf{KS Statistic} & \textbf{$p$-value $\uparrow$}\\
 & \textbf{(/95)} & \textbf{(mean $\pm$ std)} & \\
\midrule
Gemini 3.1 Pro  & 94 & $.082 \pm .003$ & $< .001$\\
Claude Opus 4.6 & 68 & $.118 \pm .001$ & $< .001$\\
Kimi-K2.5       & 29 & $.155 \pm .003$ & $< .001$\\
MiniMax-M2.5    & 95 & $.121 \pm .004$ & $< .001$\\
Qwen3.6 Plus    & 58 & $.158 \pm .011$ & $< .001$\\
Grok-4.20       & 82 & $.197 \pm .003$ & $< .001$\\
\bottomrule
\end{tabular}
}
\caption{Character-level uniformity test results. All models reject the uniform distribution, confirming systematic character-level bias.}
\label{tab:char_level}
\end{table}

\section{Proposed Method}
\label{sec:method}

We introduce \textbf{Dual-Seed Comparison} (DSC), a prompt-based protocol that extracts a pseudo-uniform variate from the character-level stochasticity of LLM text generation and uses it as a universal intermediate representation for one-dimensional sampling. Once obtained, this variate can be mapped to arbitrary target distributions via standard transforms such as inverse CDF sampling. DSC avoids interpreting a single biased string directly. Instead, it compares two strings generated under the same procedure, so shared character-level biases are largely mitigated by symmetry. DSC exploits this property through a five-stage pipeline that is executed entirely within a single LLM inference call. An example is shown in Figure~\ref{fig:main}.

Let $\mathcal{A}$ denote the set of 95 printable ASCII characters (codes 32-126).
Given a language model $\mathcal{M}$, one invocation of DSC proceeds as follows.

\paragraph{Stage 1: Dual Seed Generation.}
The model is prompted to produce two independent character sequences of length $L$:
\begin{equation}
\begin{aligned}
  s^{(1)} &= (c^{(1)}_1,\, c^{(1)}_2,\, \dots,\, c^{(1)}_{L}), \\
  s^{(2)} &= (c^{(2)}_1,\, c^{(2)}_2,\, \dots,\, c^{(2)}_{L}), \qquad
  c^{(k)}_i \in \mathcal{A}.
\end{aligned}
\end{equation}
Each seed is drawn from the token distribution of the model; no external random source is involved. The two seeds are required to be generated independently within the same context window, so that their joint distribution reflects the full stochasticity of $\mathcal{M}$.

Although the two seeds are prompted to be generated independently, in practice, the model may still introduce correlated patterns across them. To reduce the influence of position-specific generation patterns, we ask the model to separately shuffle the characters within each seed before comparison. Formally, let $\pi_1$ and $\pi_2$ denote two permutations over $\{1,\dots,L\}$ generated by the model. The shuffled seeds are
\begin{equation}
\begin{aligned}
  \tilde{s}^{(1)} &= (c^{(1)}_{\pi_1(1)},\, c^{(1)}_{\pi_1(2)},\, \dots,\, c^{(1)}_{\pi_1(L)}), \\
  \tilde{s}^{(2)} &= (c^{(2)}_{\pi_2(1)},\, c^{(2)}_{\pi_2(2)},\, \dots,\, c^{(2)}_{\pi_2(L)}).
\end{aligned}
\end{equation}
Subsequent stages operate on $\tilde{s}^{(1)}$ and $\tilde{s}^{(2)}$ rather than the original orderings.

\paragraph{Stage 2: Bitwise Ordinal Comparison.}
For each position $i \in \{1, \dots, L\}$, a binary digit is obtained by comparing the ASCII ordinal values of the corresponding characters:
\begin{equation}
  b_i =
  \begin{cases}
    1 & \text{if } \operatorname{ord}(c^{(1)}_{\pi_1(i)}) > \operatorname{ord}(c^{(2)}_{\pi_2(i)}), \\
    0 & \text{otherwise}.
  \end{cases}
  \label{eq:bit}
\end{equation}
where $\operatorname{ord}(\cdot)$ denotes the codepoint of a character under the fixed alphabet ordering. This comparison acts as a thresholding operator: irrespective of the marginal character distributions of $\tilde{s}^{(1)}$ and $\tilde{s}^{(2)}$, each $b_i$ follows a Bernoulli distribution whose parameter depends solely on the relative ordering of the two characters at position $i$. Suppose that the two characters compared at position $i$ are independent draws from the same marginal distribution $P_c$. By symmetry, the greater-than and less-than outcomes have equal probability. Since ties are deterministically mapped to $0$, we obtain
\begin{equation}
    \Pr(b_i=1)
    =
    \frac{1-\Pr(c^{(1)}_{\pi_1(i)}=c^{(2)}_{\pi_2(i)})}{2}.
\end{equation}
Therefore, the comparison bit approaches a fair Bernoulli outcome when character collisions are rare, even if $P_c$ is not uniform.

In Appendix~\ref{app:proof}, we formally analyze the bias-suppression mechanism of the ordinal comparison operator.

\paragraph{Stage 3: Binary-to-Integer Conversion.}
The $L$-bit comparison bits $\mathbf{b} = (b_1, b_2, \dots, b_{L})$ is interpreted as a big-endian unsigned integer:
\begin{equation}
    N = \sum_{i=1}^{L} b_i \cdot 2^{L-i},
    \qquad N \in \{0,1,\dots,2^L-1\}.
  \label{eq:N}
\end{equation}

\paragraph{Stage 4: Normalization.}
The integer $N$ is mapped to the unit interval:
\begin{equation}
  u = \frac{N}{2^L}, \ u \in \Bigl\{\frac{0}{2^L},\; \frac{1}{2^L},\; \dots,\; \frac{2^L - 1}{2^L}\Bigr\} \subset [0, 1).
  \label{eq:u}
\end{equation}
As a result, we can have $2^{L}$ distinct output values with a resolution of $\Delta u = 2^{-L}$. This granularity keeps the arithmetic within the range that contemporary LLMs can perform reliably in a single forward pass.

\paragraph{Stage 5: Transformation to Target Continuous Distributions.}

After Stages 1-4, DSC obtains a pseudo-uniform variate $u \in [0,1)$. 
If the target distribution is uniform on $[0,1)$, the procedure terminates at this point and directly returns $u$ as the sample. 
For other target distributions, DSC uses $u$ as the randomness source and instructs the model to convert it into a sample from the specified distribution. Conceptually, this conversion follows the inverse-transform principle: for a target distribution with cumulative distribution function $F_0$, the corresponding sample is given by the quantile mapping
\begin{equation}
d = F_0^{-1}(u).
\label{eq:icdf}
\end{equation}
In practice, the model is prompted to carry out this transformation using its own mathematical reasoning, whether the required quantile computation is available in closed form or must be approximated. 

DSC extracts randomness from the relative order of two strings. Even if both seeds share the same biased marginal distribution over $\mathcal{A}$, the comparison in Eq.~\ref{eq:bit} mitigates much of the bias: the probability $\Pr(b_i=1)$ depends on the joint distribution of $(c^{(1)}_{\pi_1(i)}, c^{(2)}_{\pi_2(i)})$ rather than on the frequencies of individual characters.

\section{Experiments}
\label{sec:experiments}

\subsection{Experimental Setup}

\paragraph{Models.}
We evaluate five language models spanning both proprietary and open-source families. The proprietary group consists of three large-scale models: Claude-opus-4-6 \citep{claude4-6-opus}, Gemini-3.1-pro-preview-thinking \citep{gemini3.1-pro}, and MiniMax-M2.5 \cite{minimax2025minimaxm25}. The open-source group includes two smaller-scale models: Qwen3.5-27B and Qwen3.5-9B \citep{qwen3.5}. We set temperature = 1.0 to maximize output entropy, and leave $\text{top}_p$ at its API default, which is 1.0 for the endpoint used.  

\paragraph{Baselines.}
We compare three approaches for generating samples from the target distributions:

\begin{itemize}

\item \textbf{Direct}: The model is prompted to output a random number drawn from the specified target distribution directly, with no intermediate steps. This serves as a na\"ive baseline that tests the raw ability of the model to produce samples from continuous distributions.

\item \textbf{SSoT} (String Seed of Thought; \citealp{misaki2025stringseedthoughtprompting}): Representing the current state-of-the-art (SOTA) baseline, SSoT prompts the model to first generate a random string to serve as an entropy seed. The model then autonomously formulates a procedure to derive a sample from the target distribution. Crucially, the exact seed-to-sample mapping is not strictly prescribed; rather, the prompt merely instructs the model to ``use the generated seed to guide any random sampling.'' Since SSoT was originally designed for discrete and diversity-oriented generation, we use it as the closest intrinsic seed-based baseline rather than as an optimized continuous-sampling method.

\item \textbf{DSC} (Dual-Seed Comparison): The proposed method described in Section~\ref{sec:method}.
\end{itemize}

To ensure a fair comparison, we fix the length of the generated strings to 16 characters ($L=16$) across DSC and SSoT. 

\begin{table}[t]
\centering
\resizebox{\linewidth}{!}{
\small
\begin{tabular}{lll}
\toprule
\textbf{Distribution} & \textbf{Parameters} & \textbf{Diagnostic Focus} \\
\midrule
Uniform     & $a=0,\; b=1$          & Bounded flatness \\
Normal    & $\mu=0,\; \sigma=1$   & Symmetric tails \\
Exponential & $\lambda=1$           & Right skew \\
Beta        & $\alpha=2,\; \beta=5$ & Bounded skew \\
Gamma       & $\alpha=2,\; \beta=2$ & Shape-scale density \\
\bottomrule
\end{tabular}
}
\caption{Target distributions and parameter settings used in the evaluation.}
\label{tab:distributions}
\end{table}

\begin{table*}[t]
\centering
\resizebox{\textwidth}{!}{%
\begin{tabular}{l|ccc|ccc|ccc|ccc|ccc}
\toprule
& \multicolumn{3}{c|}{\textbf{Uniform}} & \multicolumn{3}{c|}{\textbf{Normal}} & \multicolumn{3}{c|}{\textbf{Exponential}} & \multicolumn{3}{c|}{\textbf{Beta}} & \multicolumn{3}{c}{\textbf{Gamma}} \\
\textbf{Model} & Direct & SSoT & DSC & Direct & SSoT & DSC & Direct & SSoT & DSC & Direct & SSoT & DSC & Direct & SSoT & DSC \\
\midrule
Claude Opus 4.6
  & .717{$\pm$.015} & .241{$\pm$.033} & \textbf{.141{$\pm$.025}}
  & .653{$\pm$.012} & .476{$\pm$.015} & \textbf{.169{$\pm$.030}}
  & .595{$\pm$.024} & .373{$\pm$.017} & \textbf{.121{$\pm$.010}}
  & .489{$\pm$.070} & .274{$\pm$.020} & \textbf{.216{$\pm$.024}}
  & .558{$\pm$.000} & .905{$\pm$.008} & \textbf{.132{$\pm$.038}} \\
Gemini 3.1 Pro
  & .332{$\pm$.007} & .373{$\pm$.026} & \textbf{.082{$\pm$.016}}
  & .485{$\pm$.012} & .232{$\pm$.051} & \textbf{.095{$\pm$.005}}
  & .328{$\pm$.004} & .347{$\pm$.016} & \textbf{.082{$\pm$.015}}
  & .416{$\pm$.009} & .494{$\pm$.014} & \textbf{.090{$\pm$.025}}
  & .334{$\pm$.009} & .384{$\pm$.005} & \textbf{.112{$\pm$.018}} \\
MiniMax-M2.5
  & .303{$\pm$.038} & .117{$\pm$.004} & \textbf{.104{$\pm$.023}}
  & .185{$\pm$.057} & .159{$\pm$.028} & \textbf{.135{$\pm$.012}}
  & .310{$\pm$.024} & \textbf{.084{$\pm$.024}} & .097{$\pm$.028}
  & .302{$\pm$.009} & .226{$\pm$.017} & \textbf{.163{$\pm$.027}}
  & .192{$\pm$.004} & .195{$\pm$.031} & \textbf{.134{$\pm$.027}} \\
Qwen3.5-27B
  & .268{$\pm$.008} & .184{$\pm$.013} & \textbf{.082{$\pm$.004}}
  & .274{$\pm$.020} & .233{$\pm$.021} & \textbf{.079{$\pm$.016}}
  & .430{$\pm$.028} & .204{$\pm$.012} & \textbf{.083{$\pm$.014}}
  & .413{$\pm$.008} & .241{$\pm$.024} & \textbf{.076{$\pm$.016}}
  & .463{$\pm$.025} & .247{$\pm$.034} & \textbf{.069{$\pm$.010}} \\
Qwen3.5-9B
  & .272{$\pm$.021} & .162{$\pm$.013} & \textbf{.089{$\pm$.026}}
  & .274{$\pm$.026} & .234{$\pm$.027} & \textbf{.120{$\pm$.014}}
  & .315{$\pm$.016} & .165{$\pm$.018} & \textbf{.076{$\pm$.024}}
  & .401{$\pm$.011} & .248{$\pm$.023} & \textbf{.053{$\pm$.012}}
  & .471{$\pm$.021} & .302{$\pm$.014} & \textbf{.158{$\pm$.022}} \\
\bottomrule
\end{tabular}%
}
\caption{KS statistic ($D_n$, mean $\pm$ std over 5 repeats, $n=200$ per repeat). Lower is better. \textbf{Bold} indicates the best method for each model-distribution pair.}
\label{tab:ks_results}
\end{table*}

\begin{table*}[t]
\centering
\resizebox{\textwidth}{!}{%
\begin{tabular}{l|ccc|ccc|ccc|ccc|ccc}
\toprule
& \multicolumn{3}{c|}{\textbf{Uniform}} & \multicolumn{3}{c|}{\textbf{Normal}} & \multicolumn{3}{c|}{\textbf{Exponential}} & \multicolumn{3}{c|}{\textbf{Beta}} & \multicolumn{3}{c}{\textbf{Gamma}} \\
\textbf{Model} & Direct & SSoT & DSC & Direct & SSoT & DSC & Direct & SSoT & DSC & Direct & SSoT & DSC & Direct & SSoT & DSC \\
\midrule
Claude Opus 4.6
  & .300{$\pm$.005} & .086{$\pm$.011} & \textbf{.057{$\pm$.013}}
  & .849{$\pm$.002} & .692{$\pm$.032} & \textbf{.273{$\pm$.024}}
  & .695{$\pm$.010} & .480{$\pm$.018} & \textbf{.210{$\pm$.020}}
  & .104{$\pm$.017} & .390{$\pm$.042} & \textbf{.049{$\pm$.005}}
  & 2.120{$\pm$.002} & 3.449{$\pm$.018} & \textbf{.662{$\pm$.138}} \\
Gemini 3.1 Pro
  & .148{$\pm$.004} & .176{$\pm$.018} & \textbf{.038{$\pm$.010}}
  & .695{$\pm$.018} & .388{$\pm$.052} & \textbf{.183{$\pm$.017}}
  & .445{$\pm$.012} & .510{$\pm$.019} & \textbf{.174{$\pm$.021}}
  & .104{$\pm$.002} & .117{$\pm$.002} & \textbf{.028{$\pm$.007}}
  & 1.515{$\pm$.051} & 1.618{$\pm$.083} & \textbf{.619{$\pm$.151}} \\
MiniMax-M2.5
  & .156{$\pm$.017} & .096{$\pm$.060} & \textbf{.042{$\pm$.014}}
  & .679{$\pm$.516} & 3.004{$\pm$3.896} & \textbf{.203{$\pm$.017}}
  & .408{$\pm$.032} & 1.007{$\pm$1.276} & \textbf{.117{$\pm$.017}}
  & \textbf{.076{$\pm$.003}} & .389{$\pm$.578} & .080{$\pm$.019}
  & .989{$\pm$.036} & 1.452{$\pm$1.023} & \textbf{.734{$\pm$.162}} \\
Qwen3.5-27B
  & .118{$\pm$.007} & .081{$\pm$.014} & \textbf{.034{$\pm$.007}}
  & .412{$\pm$.026} & .396{$\pm$.034} & \textbf{.132{$\pm$.033}}
  & .602{$\pm$.026} & .253{$\pm$.034} & \textbf{.129{$\pm$.048}}
  & .103{$\pm$.001} & .060{$\pm$.003} & \textbf{.019{$\pm$.003}}
  & 1.672{$\pm$.050} & 1.089{$\pm$.101} & \textbf{.293{$\pm$.074}} \\
Qwen3.5-9B
  & .133{$\pm$.013} & .067{$\pm$.006} & \textbf{.040{$\pm$.010}}
  & .395{$\pm$.023} & .348{$\pm$.026} & \textbf{.294{$\pm$.043}}
  & .282{$\pm$.036} & .160{$\pm$.017} & \textbf{.088{$\pm$.038}}
  & .098{$\pm$.002} & .072{$\pm$.006} & \textbf{.017{$\pm$.004}}
  & 1.716{$\pm$.040} & 1.565{$\pm$.100} & \textbf{.841{$\pm$.090}} \\
\bottomrule
\end{tabular}%
}
\caption{Wasserstein distance ($W_1$, mean $\pm$ std over 5 repeats, $n=200$ per repeat). Lower is better. \textbf{Bold} indicates the best method for each model-distribution pair.}
\label{tab:wasserstein_results}
\end{table*}

\paragraph{Distribution Sampling.}
We evaluate all methods on five continuous target distributions, as summarized in Table~\ref{tab:distributions}. The suite spans uniform, symmetric, skewed, bounded, and shape-scale distributions, covering a range of sampling behaviors. 
For the uniform distribution, DSC directly treats the generated variate $u \in [0,1)$ as the output sample. For all other distributions, DSC exposes the intermediate variate $u$ and instructs the LLM to convert it into the target sample $d$ according to the target distribution, typically via the corresponding quantile transformation. Due to the page limits, we provide instruction details in Appendix~\ref{app:ds}.

\paragraph{Downstream Applications}

To further validate the practical utility of the DSC algorithm, we evaluate its performance on two downstream applications introduced by \citet{zhao2026baddice}. Due to the page limits, scenario descriptions are detailed in Appendix~\ref{app:dt}.

\begin{itemize}
    \item \textbf{MCQ Generation.}  Models are tasked with generating multiple-choice medical questions independently. To prevent positional bias, prompts explicitly require the correct answer to be uniformly distributed across options A, B, C, and D. 
    \item \textbf{Attribute-Constrained Prompt Generation.} We stress-test sampling performance in a semantically complex environment. Models generate text-to-image prompts independently, each describing a person wearing a coat. The outputs must adhere to four target distributions: Gender (Male 49.5\%, Female 50.5\%) and Race/Ethnicity (White 57.5\%, Hispanic 20.0\%, Black 12.6\%, Asian 6.5\%, Other 3.4\%), both derived from 2024 U.S. Census Bureau data; Height following a normal distribution $\mathcal{N}(169, 10^2)$ cm; and Coat Color uniformly distributed across seven categories. 
\end{itemize}

\paragraph{Settings.}
Each experiment consists of $n = 200$ independent trials per method per model. To quantify variability, we repeat each experiment $5$ times with identical configurations, yielding $200 \times 5 = 1{,}000$ total trials per condition. Repeats are executed sequentially; within each repeat, each trial corresponds to an independent API call with no shared context. We report the mean $\pm$ standard deviation of each metric across the $5$ repeats.

\subsection{Distribution Sampling Evaluation}
\label{sec:dse}

\paragraph{Results.} Tables~\ref{tab:ks_results} and~\ref{tab:wasserstein_results} report the distributional fidelity of DSC, Direct, and SSoT across five target distributions and five language models. DSC achieves the lowest mean error in 48 of the 50 metric-specific comparisons. Taking the better of Direct and SSoT as the strongest baseline in each comparison, DSC reduces the reported error by a median of 58.5\% across the 50 comparisons.

This consistent improvement is aligned with the design of DSC. Direct sampling requires the model to realize the target distribution in a single generation step, while SSoT relies on an unconstrained seed-to-sample mapping. In contrast, DSC first extracts a pseudo-uniform variate through relative character comparisons and then applies an
explicit target-distribution transformation. The ordinal comparison reduces sensitivity to absolute character-frequency preferences, and the arithmetic process makes the seed-to-sample mapping transparent and consistent.

The gains are particularly pronounced for non-uniform target
distributions. For example, for Qwen3.5-9B on $Beta(2,5)$, DSC reduces the KS statistic from $.248$ to $.053$ and $W_1$ from $.072$ to $.017$ against SSoT, corresponding to reductions of 78.6\% and 76.4\%, respectively. The
results show that the benefit of DSC extends beyond uniform sampling to bounded, skewed, and shape-scale distributions. In absolute terms, Qwen3.5-27B obtains mean KS statistics below $D_{200,0.05}=.096$ for all 5 distributions, while Gemini~3.1~Pro and Qwen3.5-9B do so for 4 and 3 distributions, respectively.

The only two exceptions occur with MiniMax-M2.5. On the Exponential distribution, DSC obtains a KS statistic of $.097$, compared with $.084$ for SSoT. On the Beta distribution, DSC obtains $W_1=.080$, slightly above the Direct result of $.076$. Both differences are small and confined to a single model. We examine the sources of these residual errors through the pipeline diagnostics described next.

\paragraph{Pipeline Diagnostics.} 
As final-sample evaluation does not identify where errors arise in DSC, we conduct a more precise analysis in Appendix~\ref{sec:ablation} and Appendix~\ref{app:bit_seed_diagnostics}. Appendix~\ref{sec:ablation} further combines stage-wise execution accuracy, which checks whether each operation is performed exactly; stage-wise KS analysis, which tracks whether each operation changes distributional fidelity; and a seed-source ablation, which isolates seed defects from execution errors. Appendix~\ref{app:bit_seed_diagnostics} further compares the original and shuffled seed pairs at the bit and seed levels.

Across 20 of the 25 model-distribution settings, the final-stage $D_n$ differs from its Stage~2 value by at most $0.02$. Thus, most deviation in distribution fidelity is already present after seed generation and comparison-bit construction, while later operations usually preserve it. The seed ablation supports this interpretation: shuffling reduces fixed-position structure, and program-generated seeds reduce the remaining error further. However, for some distributions, the nonlinear transformation in Stage~5 can introduce additional distortion in a few nonlinear-transformation settings.





\subsection{Downstream Applications}
In the downstream experiments, we evaluate task-adapted implementations of DSC rather than the complete five-stage sampling pipeline. These implementations retain the core dual-seed ordinal-comparison operator but apply it without shuffling. 

\paragraph{MCQ Generation.} 
We record the position of the correct answer for each question and conduct a $\chi^2$ goodness-of-fit test against the ideal uniform distribution (25\% per option). For the four-option MCQ task, DSC extracts only two bits from the comparison string, mapping the 2-bit code to answer letters ($00\rightarrow A$, $01\rightarrow B$, $10\rightarrow C$, $11\rightarrow D$).

\paragraph{Attribute-Constrained Prompt Generation.} 
We evaluate categorical attributes via $\chi^2$ tests and continuous height via the KS test. For the attribute-constrained task, DSC generates four independent pairs of seeds, one per attribute, each producing a pseudo-uniform variate that is mapped to the target category via cumulative thresholds. 
\begin{table}[t]
\centering
\resizebox{\columnwidth}{!}{%
\begin{tabular}{l cccc >{\columncolor{gray!12}}c >{\columncolor{gray!12}}c}
\toprule
Method & A (\%) & B (\%) & C (\%) & D (\%) & $\chi^2 \downarrow$ & $p$-value $\uparrow$ \\
\midrule
Direct & 3.3$\pm$1.2 & 20.3$\pm$2.4 & 63.9$\pm$2.7 & 12.5$\pm$1.1 & 174.3$\pm$21.7 & $<.001$ \\
SSoT & 14.9$\pm$1.8 & 32.0$\pm$3.1 & 37.6$\pm$3.8 & 15.5$\pm$1.3 & 34.3$\pm$9.9 & $<.001$ \\
DSC & 26.2$\pm$1.4 & 25.0$\pm$4.1 & 24.7$\pm$3.4 & 24.1$\pm$1.9 & \textbf{2.9$\pm$1.5} & \textbf{.458$\pm$.239} \\
\midrule
Uniform & 25.0 & 25.0 & 25.0 & 25.0 & -- & -- \\
\bottomrule
\end{tabular}
}
\caption{Distribution fidelity in MCQ generation (Qwen3.5-9B, mean $\pm$ std over 5 repeats, $n=200$ per repeat). Target: Uniform 25\% per option. \textbf{Bold} indicates the best $\chi^2$.}
\label{tab:mcq_bias}
\end{table}

\begin{table}[t]
\centering
\resizebox{\columnwidth}{!}{%
\begin{tabular}{l ccc|c}
\toprule
& \textbf{Direct} & \textbf{SSoT} & \textbf{DSC} & \textbf{Target} \\
\midrule

\multicolumn{5}{l}{\textbf{Task: Gender}} \\
\midrule
\rowcolor{gray!12}
$\chi^2 \downarrow$ & 60.2$\pm$12.4 & 88.3$\pm$15.6 & \textbf{1.5$\pm$2.5} & -- \\
Male (\%)   & 22.2$\pm$2.7 & 16.4$\pm$3.0 & 52.3$\pm$3.3 & 49.5 \\
Female (\%) & 77.8$\pm$2.7 & 83.6$\pm$3.0 & 47.7$\pm$3.3 & 50.5 \\

\midrule
\multicolumn{5}{l}{\textbf{Task: Race/Ethnicity}} \\
\midrule
\rowcolor{gray!12}
$\chi^2 \downarrow$ & 168.9$\pm$14.2 & 53.3$\pm$16.1 & \textbf{31.9$\pm$19.4} & -- \\
White (\%)    & 17.6$\pm$1.2 & 39.3$\pm$4.5 & 45.6$\pm$6.8 & 57.5 \\
Hispanic (\%) & 48.0$\pm$3.0 & 39.3$\pm$2.7 & 29.7$\pm$5.3 & 20.0 \\
Black (\%)    & 20.9$\pm$2.8 & 13.1$\pm$2.6 & 19.1$\pm$3.5 & 12.6 \\
Asian (\%)    & 13.5$\pm$1.3 & 6.5$\pm$0.8 & 4.0$\pm$1.1 & 6.5 \\
Others (\%)   & 0.0$\pm$0.0 & 1.8$\pm$0.8 & 1.6$\pm$0.8 & 3.4 \\

\midrule
\multicolumn{5}{l}{\textbf{Task: Height}} \\
\midrule
\rowcolor{gray!12}
$D_n$ $\downarrow$ & 0.167$\pm$0.012 & 0.308$\pm$0.035 & \textbf{0.135$\pm$0.029} & -- \\
\rowcolor{gray!12}
$W_1 \downarrow$ & 3.20$\pm$0.39 & 5.24$\pm$0.24 & \textbf{2.15$\pm$0.32} & -- \\
$\mu$ (cm)     & 169.6$\pm$0.5 & 171.8$\pm$0.5 & 168.8$\pm$1.5 & 169.0 \\
$\sigma$ (cm) & 6.2$\pm$0.4   & 4.9$\pm$0.6   & 11.9$\pm$2.3  & 10.0 \\

\midrule
\multicolumn{5}{l}{\textbf{Task: Coat Color}} \\
\midrule
\rowcolor{gray!12}
$\chi^2 \downarrow$ & 89.5$\pm$15.0 & 91.9$\pm$12.2 & \textbf{33.5$\pm$9.2} & -- \\
Black (\%)  & 2.6$\pm$0.5  & 6.2$\pm$1.7  & 18.5$\pm$3.5 & 14.3 \\
White (\%)  & 0.9$\pm$0.6  & 5.4$\pm$1.5  & 9.1$\pm$2.0  & 14.3 \\
Red (\%)    & 19.3$\pm$1.4 & 13.4$\pm$3.9 & 18.7$\pm$2.5 & 14.3 \\
Blue (\%)   & 23.2$\pm$2.9 & 33.7$\pm$1.2 & 21.1$\pm$3.1 & 14.3 \\
Green (\%)  & 27.5$\pm$3.4 & 21.2$\pm$4.5 & 15.5$\pm$2.3 & 14.3 \\
Yellow (\%) & 13.2$\pm$3.8 & 9.4$\pm$2.3  & 11.7$\pm$1.6 & 14.3 \\
Brown (\%)  & 13.3$\pm$2.2 & 10.7$\pm$1.6 & 5.4$\pm$1.4  & 14.3 \\

\bottomrule
\end{tabular}%
}
\caption{Distribution fidelity in attribute-constrained prompt generation (Qwen3.5-9B, mean $\pm$ std over 5 repeats, $n=200$ per repeat). Shaded rows report the primary distributional metrics, where lower values indicate closer agreement with the target distribution. \textbf{Bold} indicates the best result for each metric.}
\label{tab:joint_attr}
\end{table}

\paragraph{Results.} We use Qwen3.5-9B and compare DSC with the two baselines. Table~\ref{tab:mcq_bias} reports the answer distribution for a four-option MCQ task where the target is uniform selection. Direct prompting exhibits extreme positional bias, concentrating 63.9\% of selections on option C while nearly ignoring option A (3.3\%). SSoT reduces this imbalance but still deviates significantly from uniformity ($\chi^2=34.3$, $p<.001$). In contrast, DSC produces a near-uniform distribution across all four options, with a non-significant chi-square test ($\chi^2=2.9$, $p=0.458$), indicating that the null hypothesis of uniformity cannot be rejected.  Table~\ref{tab:joint_attr} evaluates DSC on a multi-attribute sampling task where four attributes must simultaneously follow prescribed distributions. For gender, DSC achieves near-perfect balance ($\chi^2=1.5$), while Direct and SSoT both exhibit strong female bias (78\% and 84\% respectively). For race/ethnicity, DSC attains the best performance ($\chi^2=31.9$), though all three methods still deviate from the census targets. For the continuous height attribute, DSC best recovers both the target mean ($\mu=168.8$ vs.\ target 169.0) and spread ($\sigma=11.9$ vs.\ target 10.0), yielding the lowest KS statistic (0.135) and $W_1$ (2.15). For coat color, DSC again achieves the lowest chi-square (33.5), distributing selections more evenly across all seven categories. These results demonstrate that the advantage of DSC extends beyond univariate sampling to joint multi-attribute generation, suggesting its practical potential as a lightweight debiasing mechanism for distribution-controlled generation tasks.

\section{Conclusion}

We presented DSC, a tool-free protocol enabling LLMs to natively sample from target distributions. By extracting a pseudo-uniform variate through position-wise string comparisons, DSC achieves the best KS statistic in 24 of 25 model-distribution conditions. Beyond statistical fidelity, DSC improves attribute-level control in downstream generation prompts. In MCQ generation and attribute-constrained prompt generation tasks, DSC yields attribute distributions closer to those produced by an ideal probabilistic sampler. As a natural extension of this work, future research can explore attention masking or decoupling prompts to enforce the mutual independence of seed pairs strictly, pushing the boundaries of native probabilistic sampling of LLMs.

\section*{Limitations}

While Dual-Seed Comparison (DSC) improves distributional fidelity without relying on external tools, it is subject to several practical limitations:

\paragraph{Reliance on LLM Arithmetic Capabilities.} DSC is not a pure natural language task; it requires models to accurately execute multi-step deterministic operations, including ordinal comparison, binary-to-integer conversion, and floating-point division. Models with weaker computational abilities suffer from compounding errors across intermediate stages, rendering the final sampling unreliable. Consequently, DSC is best suited for foundation models with strong reasoning and math capabilities.

\paragraph{Dependence on distribution knowledge.} For non-uniform target distributions, DSC relies on the ability of the model to evaluate or approximate the corresponding quantile function (e.g., inverse CDFs), and convert the pseudo-uniform variate into a target sample. While modern LLMs are highly proficient with standard distributions (e.g., Normal, Exponential), and the accuracy of the method may degrade when applied to obscure or highly complex distributions.

\paragraph{Inference Overhead.} Compared to direct prompting, DSC necessitates generating a substantial number of intermediate tokens (e.g., two string seeds, bits comparison, and intermediate arithmetic steps) to yield a single valid sample. In latency-sensitive or cost-constrained real-time applications, this increased computational overhead may pose a practical limitation.

\section*{Ethical Considerations}

From a risk mitigation perspective, we explicitly caution that DSC is not a cryptographically secure pseudo-random number generator (CSPRNG). DSC is designed exclusively to recover statistical distributional fidelity for applications such as agent-based simulations, synthetic data generation, and procedural content creation. Because the internal state of an LLM is deterministic and its outputs are heavily influenced by the prompt context, the generated random seeds are inherently predictable to an adversary with sufficient system access or prompt knowledge. DSC must never be deployed in security-critical contexts, such as cryptographic key generation, secure token issuance, or applications involving direct financial stakes.

\bibliography{my}

\newpage
\begin{appendices}

\section*{Appendix}
\phantomsection
\addcontentsline{toc}{section}{Appendix} 

\startcontents[appendix]
\startlist[appendix]{lof}

\printcontents[appendix]{}{1}{\section*{Table of Contents}}

\printlist[appendix]{lof}{}{\section*{List of Figures}}

\clearpage

\section{Theoretical Analysis of Ordinal Comparison}
\label{app:proof}

\subsection{Notation and Assumptions}

We analyze the ordinal comparison operator defined in Eq.~\ref{eq:bit}. 
Let $c$ denote a character sampled from the model-induced marginal distribution $P_c$ over the alphabet $\mathcal{A}$ defined in Section~4. For each $a\in\mathcal{A}$, let $p_a=P_c(a)=\Pr(c=a)$.

Define the \emph{collision probability} of $P_{c}$ as
\begin{equation}
\gamma(P_{c})\;\triangleq\;\sum_{a\in\mathcal{A}} p_a^{2},
  \tag{A1}
  \label{eq:collision}
\end{equation}
which is the probability that two independent draws from $P_{c}$ produce the same character. 
Since $|\mathcal{A}|=K$, we have $\gamma(P_{c})\geq 1/K$, with equality if and only if $P_{c}$ is uniform.

\begin{assumption}[Pairwise Independent Seeds]
\label{asm:indep}
The two seed strings used for comparison are independent and identically distributed (i.i.d.) draws from the same distribution $\mathcal{P}_{\text{str}}$ over the string space. 
\end{assumption}

This assumption is local to a single DSC invocation. Unlike a global i.i.d. requirement over all seed strings generated across all sampling trials, DSC requires only that the two seeds compared within each invocation be independent draws from the same distribution. It does not require seed strings from different invocations to be mutually independent.

In practice, characters within a generated string might exhibit local dependencies due to the autoregressive nature of LLMs. However, for mathematical tractability in deriving our theoretical bounds, we introduce the following idealized assumption:

\begin{assumption}[Positional Independence Heuristic]
\label{asm:positional}
The comparison outcomes (bits) at different positions are treated as mutually independent.
\end{assumption}

Under this framework, we write a generic comparison bit as
\begin{equation}
  b = \mathbf{1}\!\left[\operatorname{ord}(c^{(1)}) > \operatorname{ord}(c^{(2)})\right],
  \label{eq:bit_def}
  \tag{A2}
\end{equation}
where $c^{(1)}$ and $c^{(2)}$ are the characters compared at a single generic position. By Assumption~\ref{asm:indep}, because the two underlying strings are mutually independent, $c^{(1)}$ and $c^{(2)}$ are independent draws from the marginal distribution $P_{c}$. 
This is the single-position version of the DSC bit rule in Eq.~\ref{eq:bit}, with ties mapped to $0$.

\subsection{Bias of the Comparison Bit}

\begin{proposition}[Comparison Bias]
\label{prop:bias}
Under Assumption~\ref{asm:indep}, the comparison bit defined in Eq.~\ref{eq:bit_def} satisfies
\begin{align}
  \Pr(b=1)\;=\;\frac{1-\gamma(P_{c})}{2}.
  \tag{A3}
  \label{eq:pr_b}
\end{align}
\end{proposition}

\begin{proof}
Because the two strings are mutually independent (Assumption~\ref{asm:indep}), the characters $c^{(1)}$ and $c^{(2)}$ compared at any generic position are independent draws from the same marginal distribution $P_{c}$. 

The sample space decomposes into three disjoint events:
\begin{equation}
\begin{split}
  &\Pr\!\bigl(\operatorname{ord}(c^{(1)}) > \operatorname{ord}(c^{(2)})\bigr)\\
   + &\Pr\!\bigl(\operatorname{ord}(c^{(1)}) < \operatorname{ord}(c^{(2)})\bigr) \\
  + &\Pr\!\bigl(c^{(1)} = c^{(2)}\bigr)
  = 1.
\end{split}
\tag{A4}
\end{equation}
Since $c^{(1)}$ and $c^{(2)}$ are independent and identically distributed, their joint distribution is exchangeable. Hence, by symmetry:
\begin{equation}
\begin{split}
  &\Pr\!\bigl(\operatorname{ord}(c^{(1)})>\operatorname{ord}(c^{(2)})\bigr) \\
  &\qquad =
  \Pr\!\bigl(\operatorname{ord}(c^{(1)})<\operatorname{ord}(c^{(2)})\bigr).
\end{split}
\tag{A5}
\end{equation}
The tie probability is
\begin{equation}
\begin{split}
  \Pr\!\bigl(c^{(1)}=c^{(2)}\bigr)
  &=
  \sum_{a\in\mathcal{A}} \Pr(c^{(1)}=a, c^{(2)}=a) \\
  &=
  \sum_{a\in\mathcal{A}} p_a^2
  =
  \gamma(P_{c}),
\end{split}
\tag{A6}
\end{equation}
where the second equality follows from the mutual independence of $c^{(1)}$ and $c^{(2)}$. Combining the three-way decomposition with the exchangeability identity gives Eq.~\ref{eq:pr_b}.
\end{proof}

\begin{corollary}[Deviation from a Fair Coin]
\label{cor:deviation}
The deviation of $b$ from a fair coin is entirely determined by the collision probability:
\begin{equation}
  \left|\Pr(b=1)-\tfrac{1}{2}\right| \;=\; \frac{\gamma(P_{c})}{2}.
  \tag{A7}
\end{equation}
For $K=95$ printable ASCII characters under a uniform character distribution, $\gamma(P_{c})=1/95$, giving a deviation of $1/190\approx 0.0053$. 
More generally, the deviation depends on the character distribution solely through the collision probability $\gamma(P_{c})$. Notice that $\gamma(P_{c}) = \sum_{a} p_a^2 \le \max_a(p_a)$, meaning the deviation is strictly bounded by $\frac{1}{2}\max_a(p_a)$. Thus, ordinal comparison remains close to a fair coin whenever the character distribution is sufficiently diffuse (i.e., lacks heavy probability masses on any single character), even if it deviates from a strictly uniform distribution.
\end{corollary}

The deterministic tie-breaking rule in Eq.~\ref{eq:bit_def} consistently maps ties to $0$, thereby introducing a systematic bias toward $0$ of magnitude $\gamma(P_{c})/2$. While stochastic tie-breaking could theoretically eliminate this residual bias, it would necessitate an external source of randomness. Consequently, we opt for deterministic tie-breaking to preserve the self-contained reproducibility of the protocol.

\subsection{Entropy of the Comparison Bit}

\begin{proposition}[Per-Bit Entropy]
\label{prop:entropy}
Under Assumption~\ref{asm:indep}, let $H(b)$ denote the Shannon entropy of the comparison bit $b$ at any generic position. Then
\begin{equation}
  H(b) \;=\; h\!\left(\frac{1-\gamma(P_{c})}{2}\right),
  \tag{A8}
\end{equation}
where $h(q) = -q\log_2 q - (1-q)\log_2(1-q)$ is the binary entropy function.
\end{proposition}

\begin{proof}
By Proposition~\ref{prop:bias}, $b\sim\mathrm{Bernoulli}\!\bigl((1-\gamma(P_{c}))/2\bigr)$. 
The result follows directly from the entropy of a Bernoulli random variable.
\end{proof}

This result shows that the entropy of each comparison bit is governed by the collision probability $\gamma(P_{c})$. 
When the character distribution is diffuse, $\gamma(P_{c})$ is small, and the Bernoulli parameter $(1-\gamma(P_{c}))/2$ is close to $1/2$, yielding nearly one bit of entropy per comparison. 
For a uniform distribution over $K=95$ printable ASCII characters, $\gamma(P_{c})=1/95$ and $H(b)\approx 1$ bit. 
Thus, ordinal comparison can produce high-entropy bits even when the underlying character distribution is not perfectly uniform, provided that ties remain rare.

For an $L$-bit comparison sequence $\mathbf{b}=(b_1,\dots,b_{L})$, the joint entropy satisfies the standard subadditivity bound
\begin{equation}
  H(\mathbf{b}) \;\leq\; \sum_{i=1}^{L} H(b_i).
  \tag{A9}
\end{equation}
If the comparison outcomes are identically distributed with entropy $H(b)$, this becomes
\begin{equation}
  H(\mathbf{b}) \;\leq\; L\cdot H(b),
  \tag{A10}
  \label{eq:joint_entropy}
\end{equation}
with equality if the comparison outcomes are independent across positions (as formalized by our idealized heuristic in Assumption~\ref{asm:positional}). 
In practice, as previously noted, local positional dependencies in LLM-generated strings mean the bits may not be perfectly independent, potentially reducing the actual joint entropy below this theoretical upper bound.

\begin{figure*}[t]
\centering
\begin{minipage}[t]{0.32\textwidth}
    \centering
    \includegraphics[width=\linewidth]{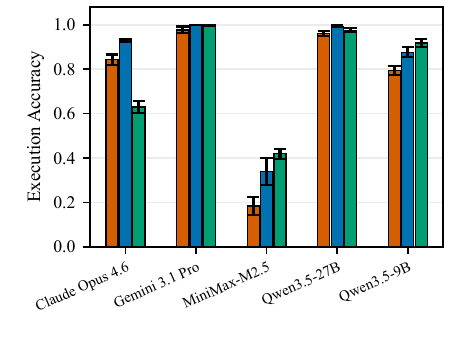}
    \subcaption{Uniform - $\mathcal{U}(0,1)$}
\end{minipage}\hfill
\begin{minipage}[t]{0.32\textwidth}
    \centering
    \includegraphics[width=\linewidth]{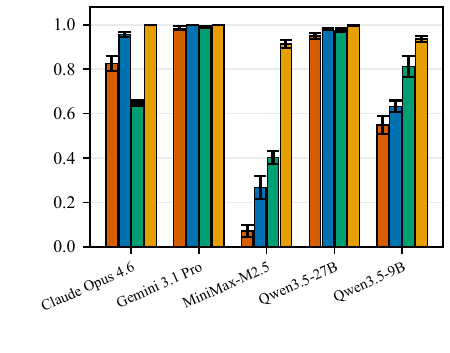}
    \subcaption{Normal - $\mathcal{N}(0,1)$}
\end{minipage}\hfill
\begin{minipage}[t]{0.32\textwidth}
    \centering
    \includegraphics[width=\linewidth]{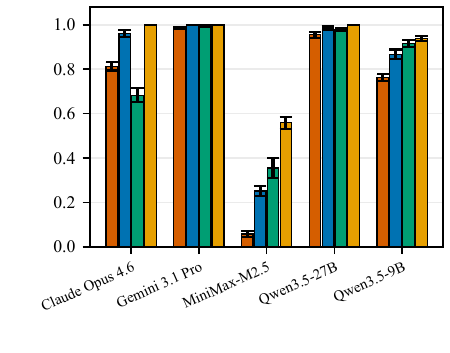}
    \subcaption{Exponential - $Exp(1)$}
\end{minipage}

\vspace{0.4em}

\makebox[\textwidth][c]{%
\begin{minipage}[t]{0.32\textwidth}
    \centering
    \includegraphics[width=\linewidth]{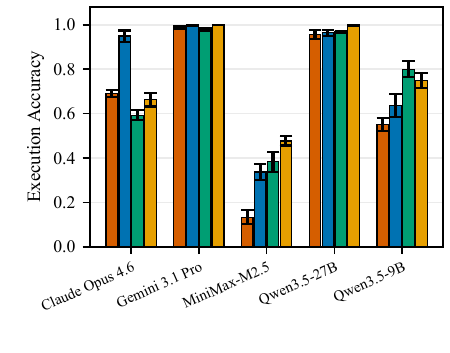}
    \subcaption{Beta - $Beta(2,5)$}
\end{minipage}
\hspace{0.04\textwidth}
\begin{minipage}[t]{0.32\textwidth}
    \centering
    \includegraphics[width=\linewidth]{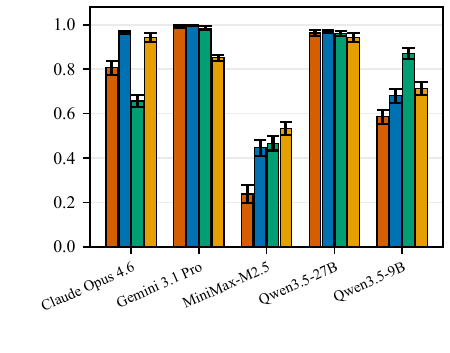}
    \subcaption{Gamma - $Gamma(2,2)$}
\end{minipage}%
}

\vspace{0.1em}
\includegraphics[width=0.57\textwidth]{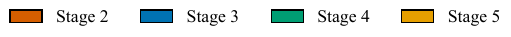}
\caption{Stage-wise execution accuracy across five target distributions. Each bar reports the mean across five repeats; error bars indicate standard deviation.}
\label{fig:process_correctness}
\end{figure*}

\subsection{Implications for DSC}

The preceding results show that bias suppression in DSC is local to each comparison pair.
The relative ordering of two characters determines each bit, and its deviation from a fair coin is controlled by the probability that the two characters collide. 
Therefore, DSC does not require the marginal character distribution to be uniform; it requires the paired characters to be sufficiently independent and unlikely to coincide. This explains the strength of DSC. By replacing absolute character identities with relative comparisons, DSC reduces sensitivity to character-frequency biases that affect raw LLM-generated strings. 




\section{DSC Pipeline Diagnostics and Seed Ablation}
\label{sec:ablation}
Section~\ref{sec:dse} evaluates DSC end to end, but final-sample statistics reveal only how far the output deviates from the target, not where that deviation enters the pipeline. We therefore trace DSC stage by stage, separating bias already present in the generated seeds and comparison bits from errors introduced by subsequent arithmetic operations.

\paragraph{Stage-wise Execution Accuracy.}
For each stage, we programmatically recompute the output using the model-reported result from the preceding stage. We mark the current stage as correct when its reported output matches this recomputed value. Conditioning on the preceding output isolates errors introduced by the current stage rather than errors propagated from earlier stages. Figure~\ref{fig:process_correctness} reports the resulting stage-wise execution accuracy.

We identify distinct patterns across models.
\begin{itemize}
    \item \textbf{Claude Opus 4.6} converts comparison bits to the integer $N$ reliably, with Stage~3 accuracy between 92.9\% and 96.5\%. Its main deficiencies occur in ordinal comparison and normalization: Stage~2 accuracy ranges from 69.2\% to 84.2\%, while Stage~4 accuracy ranges from 59.2\% to 68.2\%. Also, its Stage~5 accuracy falls to 66.3\% for Beta sampling. Therefore, this model has difficulty maintaining exact character comparisons and decimal arithmetic, with an additional weakness in the Beta quantile transformation.
    \item \textbf{Gemini 3.1 Pro} exceeds 97.7\% accuracy at Stages~2--4 for every distribution. Its failure is concentrated in the Gamma transformation, where Stage~5 accuracy drops to 85.2\%. The remaining errors are concentrated in the nonlinear transformation stage for Gamma.
    \item \textbf{Qwen3.5-27B} is the most consistently reliable model: Stage~2 accuracy is at least 95.0\%, and subsequent stages remain above 94.2\% across all distributions. 
    \item \textbf{Qwen3.5-9B} exhibits errors earlier in the pipeline. Stage~2 accuracy ranges from 54.8\% to 79.4\%, and Stage~3 accuracy ranges from 63.3\% to 87.7\%. Its Stage~5 accuracy further decreases to 75.1\% for Beta and 71.4\% for Gamma. These results identify both ordinal comparison and complex quantile computation as bottlenecks for this model.
    \item \textbf{MiniMax-M2.5} fails broadly across the deterministic pipeline. Stage~2 accuracy ranges from 5.8\% to 23.7\%, Stage~3 from 25.1\% to 44.6\%, and Stage~4 from 35.5\% to 46.7\%. MiniMax-M2.5 frequently fails the basic comparison and conversion operations required to construct the pseudo-uniform variate.
\end{itemize}

Overall, execution reliability varies widely across models. However, low execution accuracy does not necessarily imply poor distributional fidelity: an intermediate value may be numerically incorrect yet leave the empirical distribution of the final samples nearly unchanged. We will examine this in the following experiments.

\begin{table*}[!t]
\centering

\begin{minipage}[t]{.485\textwidth}
\centering
\resizebox{\linewidth}{!}{%
\begin{tabular}{lccc}
\toprule
Model & Stage 2 ($\mathbf{b}\to u$) & Stage 3 ($N\to u$) & Stage 4 ($u$) \\
\midrule
Claude Opus 4.6 & .142{$\pm$.024} & .141{$\pm$.025} & .141{$\pm$.025} \\
Gemini 3.1 Pro & .082{$\pm$.016} & .082{$\pm$.016} & .082{$\pm$.016} \\
MiniMax-M2.5 & .106{$\pm$.016} & .101{$\pm$.028} & .104{$\pm$.023} \\
Qwen3.5-27B & .082{$\pm$.004} & .082{$\pm$.004} & .082{$\pm$.004} \\
Qwen3.5-9B & .086{$\pm$.016} & .080{$\pm$.024} & .089{$\pm$.026} \\
\bottomrule
\end{tabular}
}
\captionof{table}{Stage-wise KS statistic $D_n$ for $\mathcal{U}(0,1)$ sampling. The Uniform target has no inverse-CDF stage.}
\label{tab:step_ks_random}
\end{minipage}
\hfill
\begin{minipage}[t]{.485\textwidth}
\centering
\resizebox{\linewidth}{!}{%
\begin{tabular}{lcccc}
\toprule
Model & Stage 2 ($\mathbf{b}\to u$) & Stage 3 ($N\to u$) & Stage 4 ($u$) & Stage 5 ($d$) \\
\midrule
Claude Opus 4.6 & .169{$\pm$.030} & .170{$\pm$.030} & .170{$\pm$.029} & .169{$\pm$.030} \\
Gemini 3.1 Pro & .095{$\pm$.005} & .095{$\pm$.005} & .095{$\pm$.005} & .095{$\pm$.005} \\
MiniMax-M2.5 & .141{$\pm$.020} & .145{$\pm$.010} & .144{$\pm$.013} & .135{$\pm$.012} \\
Qwen3.5-27B & .079{$\pm$.016} & .079{$\pm$.016} & .079{$\pm$.016} & .079{$\pm$.016} \\
Qwen3.5-9B & .097{$\pm$.010} & .100{$\pm$.017} & .123{$\pm$.016} & .120{$\pm$.014} \\
\bottomrule
\end{tabular}
}
\captionof{table}{Stage-wise KS statistic $D_n$ for $\mathcal{N}(0,1)$ sampling.}
\label{tab:step_ks_normal}
\end{minipage}

\vspace{0.6em}

\begin{minipage}[t]{.485\textwidth}
\centering
\resizebox{\linewidth}{!}{%
\begin{tabular}{lcccc}
\toprule
Model & Stage 2 ($\mathbf{b}\to u$) & Stage 3 ($N\to u$) & Stage 4 ($u$) & Stage 5 ($d$) \\
\midrule
Claude Opus 4.6 & .121{$\pm$.010} & .121{$\pm$.010} & .121{$\pm$.010} & .121{$\pm$.010} \\
Gemini 3.1 Pro & .082{$\pm$.015} & .082{$\pm$.015} & .082{$\pm$.015} & .082{$\pm$.015} \\
MiniMax-M2.5 & .112{$\pm$.019} & .104{$\pm$.018} & .101{$\pm$.017} & .097{$\pm$.028} \\
Qwen3.5-27B & .082{$\pm$.014} & .083{$\pm$.014} & .083{$\pm$.014} & .083{$\pm$.014} \\
Qwen3.5-9B & .074{$\pm$.027} & .062{$\pm$.025} & .065{$\pm$.023} & .076{$\pm$.024} \\
\bottomrule
\end{tabular}
}
\captionof{table}{Stage-wise KS statistic $D_n$ for $\operatorname{Exp}(1)$ sampling.}
\label{tab:step_ks_exponential}
\end{minipage}
\hfill
\begin{minipage}[t]{.485\textwidth}
\centering
\resizebox{\linewidth}{!}{%
\begin{tabular}{lcccc}
\toprule
Model & Stage 2 ($\mathbf{b}\to u$) & Stage 3 ($N\to u$) & Stage 4 ($u$) & Stage 5 ($d$) \\
\midrule
Claude Opus 4.6 & .120{$\pm$.016} & .120{$\pm$.016} & .120{$\pm$.016} & .216{$\pm$.024} \\
Gemini 3.1 Pro & .090{$\pm$.025} & .090{$\pm$.025} & .090{$\pm$.025} & .090{$\pm$.025} \\
MiniMax-M2.5 & .098{$\pm$.018} & .113{$\pm$.011} & .113{$\pm$.011} & .163{$\pm$.027} \\
Qwen3.5-27B & .077{$\pm$.018} & .077{$\pm$.017} & .077{$\pm$.017} & .076{$\pm$.016} \\
Qwen3.5-9B & .096{$\pm$.019} & .095{$\pm$.026} & .123{$\pm$.024} & .053{$\pm$.012} \\
\bottomrule
\end{tabular}
}
\captionof{table}{Stage-wise KS statistic $D_n$ for $\operatorname{Beta}(2,5)$ sampling.}
\label{tab:step_ks_beta}
\end{minipage}

\vspace{0.6em}

\begin{minipage}[t]{.485\textwidth}
\centering
\resizebox{\linewidth}{!}{%
\begin{tabular}{lcccc}
\toprule
Model & Stage 2 ($\mathbf{b}\to u$) & Stage 3 ($N\to u$) & Stage 4 ($u$) & Stage 5 ($d$) \\
\midrule
Claude Opus 4.6 & .133{$\pm$.035} & .133{$\pm$.035} & .133{$\pm$.035} & .132{$\pm$.038} \\
Gemini 3.1 Pro & .105{$\pm$.012} & .105{$\pm$.012} & .105{$\pm$.012} & .112{$\pm$.018} \\
MiniMax-M2.5 & .117{$\pm$.025} & .116{$\pm$.027} & .115{$\pm$.026} & .134{$\pm$.027} \\
Qwen3.5-27B & .068{$\pm$.016} & .068{$\pm$.016} & .068{$\pm$.016} & .069{$\pm$.010} \\
Qwen3.5-9B & .099{$\pm$.009} & .098{$\pm$.023} & .113{$\pm$.019} & .158{$\pm$.022} \\
\bottomrule
\end{tabular}
}
\captionof{table}{Stage-wise KS statistic $D_n$ for $\operatorname{Gamma}(2,2)$ sampling.}
\label{tab:step_ks_gamma}
\end{minipage}
\end{table*}

\paragraph{Stage-wise Distributional Fidelity.}
Stage-wise execution accuracy records whether each intermediate output is correct, but it does not measure how much an incorrect value affects distributional fidelity. A small rounding error may leave the empirical distribution nearly unchanged. In contrast, a large calculation error may shift the generated samples and increase their deviation from the target distribution. To localize such distributional error, we express each stage output on the unit interval:

Stage~2: $\mathbf{b}$ maps to its normalized binary value;

Stage~3: $N$ maps to $\frac{N}{2^L}$;

Stage~4: $u$ maps to itself;

Stage~5: $d$ maps to $F_0(d)$.

We then compute $D_n$ against $\mathcal{U}(0,1)$ at every stage; changes between adjacent stages show whether an operation introduces, preserves, or offsets distributional distortion. 

Tables~\ref{tab:step_ks_exponential}-\ref{tab:step_ks_gamma} show that most of the final deviation is already present in the Stage~2 comparison bits. In 20 of the 25 model-distribution settings, the final-stage $D_n$ differs from the Stage~2 value by at most $.02$. MiniMax-M2.5 makes this distinction particularly clear: despite its low execution accuracy, the absolute differences between its Stage~2 and final-stage $D_n$ values are only $.002$, $.006$, $.015$, and $.017$ for Uniform, Normal, Exponential, and Gamma sampling, respectively. Thus, many execution errors do not materially change $D_n$ relative to the deviation already present at Stage~2.

In most settings, later-stage computation changes $D_n$ little; the main exceptions occur in nonlinear target transformations. For Beta sampling, Stage~5 increases $D_n$ from $.120$ to $.216$ for Claude Opus~4.6, whereas the value for Qwen3.5-9B decreases from $.123$ to $.053$. The latter decrease reflects an interesting observation: arithmetic errors partially offset the bias already present at Stage~4 rather than improve execution accuracy.

Overall, the stage-wise KS results identify the seed-derived comparison bits as the main source of distributional error. Once the model has generated the seed pair and formed the Stage~2 comparison bits, most of the final deviation is already visible. The subsequent stages usually change $D_n$ only slightly. Difficult Stage~5 transformations, such as Beta and Gamma, can introduce additional distributional distortion, but these cases are secondary to the seed-to-bit bias.


\begin{figure*}[t]
\centering
\includegraphics[width=\textwidth]{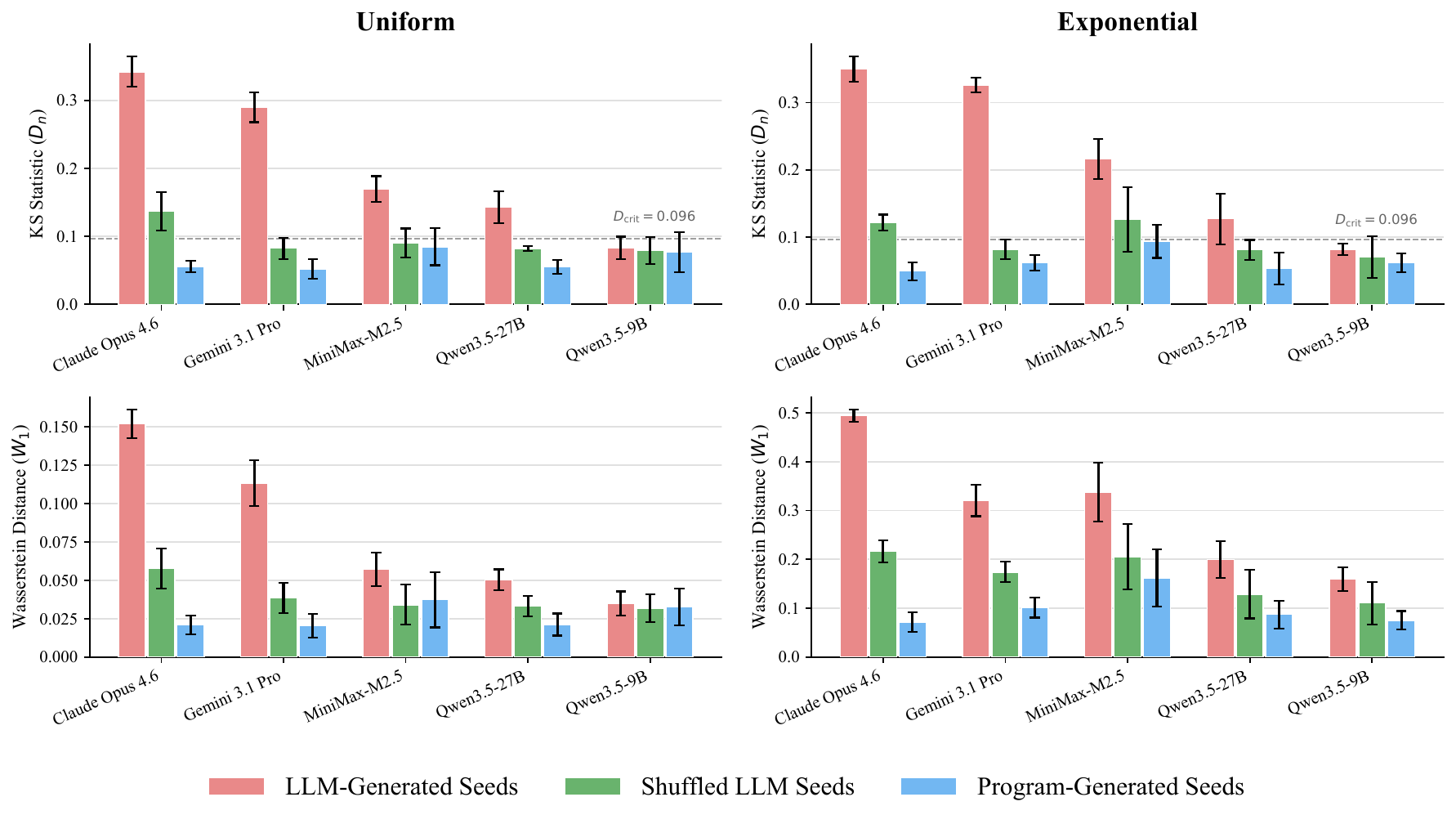}
\caption{Seed-source ablation on the Uniform and Exponential distributions, comparing original LLM seeds, shuffled LLM seeds, and program-generated seeds from Python \texttt{random.SystemRandom}. Results are mean $\pm$ standard deviation over five repeats with $n=200$ samples per repeat. The top row reports $D_n$, the bottom row reports $W_1$, and the dashed line marks $D_{\mathrm{crit}}=1.36/\sqrt{200}\approx.096$ at $\alpha=.05$.}
\label{fig:ablation_seeds}
\end{figure*}

\paragraph{Seed-Source Ablation.}
\label{app:ablation}

The execution audit does not isolate defects in the seed strings themselves. We therefore compare original LLM seed pairs, the corresponding LLM-shuffled pairs, and seed pairs generated by Python's \texttt{random.SystemRandom}. For all three conditions, comparison bits and final samples are reconstructed programmatically with the same DSC mapping. The ablation measures seed-source quality without confounding it with model-reported comparison or arithmetic errors.

Figure~\ref{fig:ablation_seeds} reveals model-specific seed defects. Claude Opus~4.6 and Gemini~3.1~Pro are affected most strongly by the original seed ordering. On Uniform sampling, shuffling reduces the KS statistic from $.342$ to $.137$ for Claude Opus~4.6 and from $.290$ to $.083$ for Gemini~3.1~Pro; on Exponential sampling, the corresponding reductions are $.350$ to $.122$ and $.326$ to $.082$. Qwen3.5-27B exhibits the same failure at a smaller magnitude, while Qwen3.5-9B changes only modestly after shuffling because its original-seed error is already comparatively low. MiniMax-M2.5 benefits from shuffling on both distributions, but its remaining error is still larger than other models, indicating that reordering alone does not remove all deficiencies in its generated seed pairs.

Program-generated seeds serve as an independent external seed source. They produce the lowest KS statistic in every model-distribution setting, and yield $D_n$ below the critical threshold in every tested setting. For example, the $D_n$ of Claude Opus~4.6 reaches $.055$ on Uniform and $.049$ on Exponential sampling with program-generated seeds. The Wasserstein results follow the same overall pattern.

Taken together, the ablation identifies LLM seed generation as a major source of residual distributional deviation. Comparing original and shuffled LLM seeds shows that shuffling lowers $D_n$ for most models, especially Claude Opus~4.6 and Gemini~3.1~Pro. However, the shuffled LLM seed source still underperforms the independent program-generated seed source. Thus, shuffling improves the seeds produced by LLMs, but it does not fully close the gap to an external random source.

\section{Diagnosing Seed Assumptions in DSC}
\label{app:bit_seed_diagnostics}

Appendix~\ref{app:proof} shows that the bias-suppression property of ordinal comparison is conditional on the seed-generation process: the paired seeds should behave as i.i.d. draws from a common distribution, and the induced comparison bits should not contain strong positional dependence. Appendix~\ref{sec:ablation} further shows that seed quality is a major source of residual error. In the seed-source ablation, shuffled seeds improve distributional fidelity over original (pre-shuffle) seeds, while replacing them with program-generated seeds further reduces the remaining error. These results raise two related questions.

\noindent \textbf{Q1: To what extent do LLM-generated seed pairs satisfy the i.i.d. property required by the theoretical analysis? }

\noindent \textbf{Q2: If original seeds deviate from these assumptions, does the shuffle step mitigate the relevant violations?}

We answer these questions by comparing diagnostics computed from the original seeds and the shuffled seeds actually used by DSC. This comparison allows us to separate two issues: whether LLMs generate seed pairs that are close to the idealized i.i.d. model, and whether shuffling improves the intermediate seed-to-bit source used by DSC.

We conduct bit-level and seed-level diagnostics using the same DSC trials as in the main distribution-sampling experiments. We aggregate diagnostics across the five target distributions. All diagnostics in this section are computed programmatically from the parsed seed strings. The diagnostics are divided into two groups. 


\begin{table*}[t]
\centering
\small
\resizebox{\textwidth}{!}{%
\begin{tabular}{lrrrrrrr}
\toprule
Model & $\hat{p}_1$ & $\mathrm{Bit}_{bias}\downarrow$ & $\widehat{\tau}$ $\downarrow$ & $\widehat{\gamma}(P_c)\downarrow$ & $\widehat{O}_{\mathrm{seed}}$ & Dup-4 $\downarrow$ & Dup-8 $\downarrow$ \\
\midrule
Claude Opus 4.6 & .494{$\pm$.008} & .008{$\pm$.006} & .000{$\pm$.000} & .022{$\pm$.000} & .001{$\pm$.001} & .053{$\pm$.021} & .003{$\pm$.004} \\
Gemini 3.1 Pro & .494{$\pm$.006} & .007{$\pm$.005} & .000{$\pm$.000} & .020{$\pm$.000} & .005{$\pm$.001} & .001{$\pm$.002} & .000{$\pm$.000} \\
MiniMax-M2.5 & .485{$\pm$.012} & .017{$\pm$.009} & .018{$\pm$.008} & .021{$\pm$.001} & .162{$\pm$.015} & .081{$\pm$.079} & .072{$\pm$.077} \\
Qwen3.5-27B & .488{$\pm$.007} & .012{$\pm$.006} & .002{$\pm$.001} & .021{$\pm$.000} & .035{$\pm$.006} & .001{$\pm$.002} & .000{$\pm$.000} \\
Qwen3.5-9B & .489{$\pm$.009} & .011{$\pm$.008} & .019{$\pm$.005} & .023{$\pm$.000} & .213{$\pm$.013} & .002{$\pm$.002} & .001{$\pm$.002} \\
\bottomrule
\end{tabular}}
\caption{
Diagnostics for shuffled seed pairs used in DSC comparison, aggregated across 5 target distributions. 
Values are reported as mean$\pm$std over 25 distribution-repeat blocks, corresponding to 5 target distributions and 5 repeats per distribution, with up to $n=200$ samples per repeat.
}
\label{tab:bit_seed_diagnostics}
\end{table*}

\begin{table*}[t]
\centering
\small
\resizebox{\textwidth}{!}{%
\begin{tabular}{lrrrrrrr}
\toprule
Model & $\hat{p}_1$ & $\mathrm{Bit}_{bias}\downarrow$ & $\widehat{\tau}$ $\downarrow$ & $\widehat{\gamma}(P_c)\downarrow$ & $\widehat{O}_{\mathrm{seed}}$ & Dup-4 $\downarrow$ & Dup-8 $\downarrow$ \\
\midrule
Claude Opus 4.6 & .511{$\pm$.013} & .015{$\pm$.009} & .002{$\pm$.004} & .022{$\pm$.000} & .002{$\pm$.004} & .534{$\pm$.039} & .264{$\pm$.030} \\
Gemini 3.1 Pro & .499{$\pm$.006} & .005{$\pm$.003} & .000{$\pm$.000} & .020{$\pm$.000} & .005{$\pm$.001} & .206{$\pm$.029} & .038{$\pm$.009} \\
MiniMax-M2.5 & .482{$\pm$.016} & .020{$\pm$.014} & .015{$\pm$.006} & .021{$\pm$.001} & .159{$\pm$.019} & .320{$\pm$.087} & .197{$\pm$.106} \\
Qwen3.5-27B & .480{$\pm$.007} & .020{$\pm$.007} & .002{$\pm$.001} & .021{$\pm$.000} & .035{$\pm$.006} & .235{$\pm$.026} & .076{$\pm$.016} \\
Qwen3.5-9B & .488{$\pm$.010} & .014{$\pm$.007} & .016{$\pm$.003} & .023{$\pm$.000} & .205{$\pm$.012} & .077{$\pm$.017} & .004{$\pm$.004} \\
\bottomrule
\end{tabular}}
\caption{
Diagnostics for original seed pairs, aggregated across 5 target distributions. Values are reported as mean$\pm$std over 25 distribution-repeat blocks, corresponding to 5 target distributions and 5 repeats per distribution, with up to $n=200$ samples per repeat.
}
\label{tab:original_bit_seed_diagnostics}
\end{table*}

The first group audits the comparison bits produced by DSC as follows: 

\paragraph{Aggregate comparison-bit balance.}
For each repeat, let $b_{t,i}$ denote the comparison
bit reconstructed from the $t$-th DSC trial at position
$i$, where $t\in\{1,\ldots,n\}$ and $i\in\{1,\ldots,L\}$. We compute the aggregate proportion of $b_{t,i} =1 $ as
\begin{equation}
    \begin{split}
    \hat{p}_1
    &=
    \frac{1}{nL}
    \sum_{t=1}^{n}\sum_{i=1}^{L} b_{t,i},
    \\
    b_{t,i}
    &=
    \mathbb{I}\!\left\{
    \operatorname{ord}(s^{(1)}_{t,i}) >
    \operatorname{ord}(s^{(2)}_{t,i})
    \right\}.
    \end{split}
    \tag{C1}
\end{equation}
Here, $n$ is the number of DSC trials within a repeat, and $L$ is the seed length. $s^{(1)}_{t,i}$ and
$s^{(2)}_{t,i}$ denote the characters compared at position $i$ in the $t$-th valid DSC trial.
We additionally report the aggregate comparison-bit bias
\begin{equation}
    \mathrm{Bit}_{bias}
    =
    \left|\hat{p}_1-\frac{1}{2}\right|.
    \tag{C2}
\end{equation}
A value of $\hat{p}_1$ close to $0.5$ indicates that, after pooling all
calls and comparison positions within a block, neither comparison
direction systematically dominates. This statistic evaluates aggregate
bit balance; it does not establish independence across positions.

\paragraph{Empirical tie rate.}
For each repeat, we compute the empirical frequency of position-wise equal characters as
\begin{equation}
    \widehat{\tau}
    =
    \frac{1}{nL}
    \sum_{t=1}^{n}\sum_{i=1}^{L}
    \mathbb{I}\!\left\{
    s^{(1)}_{t,i}=s^{(2)}_{t,i}
    \right\}.
    \tag{C3}
\end{equation}
Because DSC deterministically maps an equal-character comparison to $0$, every tie contributes a zero bit. Under the exchangeability condition used in Appendix~\ref{app:proof}, $\Pr(b=1)=(1-\tau)/2$; hence, a higher tie probability produces a larger downward deviation from $0.5$. 


The second group audits the generated seed strings themselves, including: 
\paragraph{Estimated collision probability.}
For each repeat, we estimate the common marginal character distribution by pooling the characters from both seeds:
\begin{equation}
\begin{split}
    &\widehat{p}_a
    =
    \frac{1}{2nL}
    \sum_{t=1}^{n}\sum_{i=1}^{L}
    \left[
    \mathbb{I}\{s^{(1)}_{t,i}=a\}
    +
    \mathbb{I}\{s^{(2)}_{t,i}=a\}
    \right], \\
    &\widehat{\gamma}(P_c)
    =
    \sum_{a\in\mathcal{A}}\widehat{p}_a^{\,2}.
\end{split}
\tag{C4}
\end{equation}
$\widehat{\gamma}(P_c)$ estimates the probability that two independent characters drawn from the pooled marginal character distribution are equal. Under the idealized assumption that the two compared characters are independent draws from this common distribution, the expected tie rate equals $\widehat{\gamma}(P_c)$. Therefore, a difference between $\widehat{\tau}$ and
$\widehat{\gamma}(P_c)$ means that the two seeds contain identical characters at corresponding comparison positions more or less often than expected under the independent-draw model.

\paragraph{Seed-pair character overlap.}
Seed overlap measures how much character content is shared by the two seeds in a pair, irrespective of character positions. For a seed $s$, let
\begin{equation*}
    m_a(s)=\sum_{i=1}^{L}\mathbb{I}\{s_i=a\}
\end{equation*}
denote the multiplicity of character $a$ in $s$. 

For each repeat, we compute
\begin{equation}
    \widehat{O}_{\mathrm{seed}}
    =
    \frac{1}{n}
    \sum_{t=1}^{n}
    \frac{
        \sum_{a\in\mathcal{A}}
        \min\!\left\{
            m_a(s^{(1)}_t),
            m_a(s^{(2)}_t)
        \right\}
    }{L}.
    \tag{C5}
\end{equation}
The statistic ranges from $0$ to $1$. A value of $0$ means that the two seeds share no characters, whereas a value of $1$ means that they have identical character multisets. Unlike the empirical tie rate $\widehat{\tau}$, $\widehat{O}_{\mathrm{seed}}$ ignores character
positions and measures only the reuse of characters within each seed pair.

\paragraph{Repeated-prefix rate.}
$Dup$-$4$ and $Dup$-$8$ measure the fraction of non-unique 4-character and 8-character prefixes among the generated seeds. High repeated-prefix rates indicate template-like seed generation across independent API calls.

\paragraph{A1: LLM-generated seeds are imperfect but sufficient for near-balanced comparison bits.}
Tables~\ref{tab:bit_seed_diagnostics} and ~\ref{tab:original_bit_seed_diagnostics} show that LLM-generated seeds approximate the conditions needed for DSC bias suppression, while exhibiting measurable departures from the idealized i.i.d. model.

The marginal character distributions are non-uniform, with $\widehat{\gamma}(P_c)$ ranging from $.020$ to $.023$, compared with the uniform printable-ASCII baseline of $1/95\approx.0105$. This indicates that LLMs concentrate probability mass on particular characters, consistent with the character-level analysis in Section~\ref{sec:llmsfail}.

Furthermore, paired seeds do not always behave like independent draws from the same distribution. Under the idealized independent-draw model, the empirical tie rate $\widehat{\tau}$ should be close to $\widehat{\gamma}(P_c)$. For Claude Opus 4.6, Gemini 3.1 Pro, and Qwen3.5-27B, the observed tie rates are lower than the corresponding collision estimates. Thus, equal characters occur at corresponding comparison positions much less frequently than expected under the independent-draw model. This discrepancy may reflect a model tendency to generate visibly distinct seed pairs or differences in character distributions across positions. In contrast, the tie rates of MiniMax-M2.5 and Qwen3.5-9B are closer to their estimated collision probabilities and are more consistent with the independent-draw prediction at
the single-position level. Their larger $\widehat{O}_{\mathrm{seed}}$ values indicate greater character sharing within each pair, but overlap alone does not establish dependence. Thus, agreement with the idealized independent-draw model varies across models.

However, the induced comparison bits are nevertheless close to balanced in aggregate. After shuffling, $\hat{p}_1$ lies between $.485$ and $.494$, with $\mathrm{Bit}_{bias}$ below $.017$ for all models. This supports the intuition of our approach: ordinal comparison can suppress much of the character-level bias even when the seed-generation process is imperfect. 

\paragraph{A2: Shuffling reduces fixed-position structure and improves aggregate comparison-bit balance.}

Comparing Table~\ref{tab:bit_seed_diagnostics} with Table~\ref{tab:original_bit_seed_diagnostics} shows that shuffling mainly mitigates prefix-level and position-specific templates in original seeds. Since DSC compares seeds position by position, repeated beginnings in seeds can propagate into the comparison bits. However, shuffling relocates characters within each seed before comparison, thereby reducing this fixed-position structure.

This effect is clear in the repeated-prefix diagnostics. For Claude Opus 4.6, $Dup-4$ drops from $.534$ to $.053$, and $Dup-8$ drops from $.264$ to $.003$. Gemini 3.1 Pro decreases from $.206$ to $.001$ in $Dup-4$ and from $.038$ to nearly zero in $Dup-8$. Shuffling also improves aggregate comparison-bit balance for some models; for example, Claude Opus 4.6 moves from $\hat{p}_1=.511$ to $.494$. However, shuffling does not make the seeds fully i.i.d. Because it only permutes characters within each seed, the pooled character distribution is largely preserved, so $\widehat{\gamma}(P_c)$ remains nearly unchanged. 

Thus, the answer to Q2 is affirmative but qualified. Shuffling does not make the paired seeds independent, nor does it alter their character distributions. It mitigates the violations most relevant to position-wise comparison by disrupting repeated prefixes and fixed-position regularities. 

\paragraph{Conclusion.} LLM-generated seed pairs do not fully satisfy the i.i.d. assumptions required by the theoretical analysis. However, the shuffle step mitigates the positional bias of LLM-generated seeds by disrupting repeated prefixes and improving aggregate comparison-bit balance. This explains the ablation results: shuffled LLM seeds improve over raw LLM seeds because shuffling reduces fixed-position structure in the original seeds and produces more balanced comparison bits.

\onecolumn
\section{Prompt Used in Distribution Sampling}
\label{app:ds}

\begin{exercisebox}[Listing D.1: DSC System Prompt (Distribution Sampling)]

You are a random number generator using dual-seed comparison.

\#\# Algorithm

1. Generate two independent 16-character random strings ($\text{seed}_1$, $\text{seed}_2$). Use printable ASCII characters for both seeds, and make them look random and irregular. The two seeds are completely independent — generate $\text{seed}_2$ without copying, mirroring, or aligning it to $\text{seed}_1$. 
   
2. Independently shuffle the characters within each seed: put the original $\text{seed}_1$ into a new random order, and put the original $\text{seed}_2$ into its own separate random order (a different order from $\text{seed}_1$). Each shuffled seed must be a rearrangement of exactly the same characters of its original. From here on, $\text{seed}_1$ and $\text{seed}_2$ refer to these SHUFFLED strings — use them for every step below.
   
3. For each position $i$ ($0,\dots,15$), compare ASCII values and obtain a binary string bit[]:

   \qquad - If $\text{seed}_1$[i] $>$ $\text{seed}_2$[i]: bit[i] = 1
   
   \qquad - If $\text{seed}_1$[i] $\leq$ $\text{seed}_2$[i]: bit[i] = 0
   
4. Convert the 16-bit binary string to decimal integer $N$.

5. Compute $u = N / 65536$ (i.e. $N / 2^{16}$). This gives a float number in $[0, 1)$.

\textbf{Normal Distribution}: [6. Convert $u$ to a standard normal random variable $z$ using the inverse CDF: $z$ is the value where $P(Z \leq z) = u$ for $Z \sim \mathcal{N}(0,1)$. Use your best knowledge of the normal distribution to compute $z$.]

\textbf{Exponential Distribution}: [6. Convert $u$ to an exponential random variable $z$ using the inverse CDF: $z$ is the value where $P(Z \leq z) = u$ for $Z \sim Exp(1)$. Use your best knowledge of the exponential distribution to compute $z$.]

\textbf{Beta Distribution}: [6. Convert $u$ to a Beta random variable $z$ using the inverse CDF: $z$ is the value where $P(Z \leq z) = u$ for $Z \sim Beta(2, 5)$. Use your best knowledge of the Beta distribution to compute $z$.]

\textbf{Gamma Distribution}: [6. Convert $u$ to a Gamma random variable $z$ using the inverse CDF: $z$ is the value where $P(Z \leq z) = u$ for $Z \sim Gamma(2, 2)$. Use your best knowledge of the Gamma distribution to compute $z$.]

You must perform all calculations yourself. Show your work.

Output format (strict):

\qquad <seed1>16 characters</seed1>

\qquad <seed2>16 characters</seed2>

\qquad <bits>16 binary digits</bits>

\qquad <N>decimal integer</N>

\qquad <u>decimal between 0 and 1</u>

\qquad <z>final value</z>

\end{exercisebox}

\begin{exercisebox}[Listing D.2: Direct System Prompt (Distribution Sampling)]
You are a true random number generator.

Your task: output a single random number drawn from  

\textbf{Uniform Distribution}: [the range $[0, 1)$.]

\textbf{Normal Distribution}: [STANDARD NORMAL distribution $N(0, 1)$.]

\textbf{Exponential Distribution}: [EXPONENTIAL distribution $Exp(1)$.]

\textbf{Beta Distribution}: [BETA distribution $Beta(2, 5)$.]
 
\textbf{Gamma Distribution}: [GAMMA distribution $Gamma(2, 2)$.]

Do not use any fixed or predictable pattern. Make every digit as random as possible, respecting the
\textcolor{blue}{[corresponding probability density]}.

Output ONLY the number wrapped in XML tags: \qquad <value>YOUR NUMBER</value>
\end{exercisebox}

\begin{exercisebox}[Listing D.3: SSoT System Prompt (Distribution Sampling)]
You are a helpful AI Assistant designed to provide well-reasoned and detailed responses. If the task involves probabilistic or non-deterministic reasoning, you must begin by generating a unique and complex random string with exactly 16 characters to serve as a seed.

This random string should appear sufficiently complex and unpredictable, with no obvious structure or pattern. Use your judgment to ensure it looks arbitrary and unguessable.

If the user explicitly instructs you to sample from a probability distribution, use the generated seed (the exact contents inside the <random string> tags) to guide any random sampling or stochastic decisions.

Follow these three steps for every instruction:

\quad 1. Output the random 16-character string enclosed within <random string> and </random string> tags.

\quad 2. Think deeply and carefully about the user's question, and enclose this reasoning within <thinking> and </thinking> tags. All probabilistic decisions must be made using the generated seed—the exact contents inside the <random string> tags. Make sure to extract maximum randomness from the string by using all of its content.

\quad 3. Provide your final answer, enclosed within <answer> and </answer> tags.

Strictly follow this tag structure, and respond in the following format, do NOT response in other formats:

\qquad <random string>...</random string>

\qquad <thinking>...</thinking>

\qquad <answer>...</answer>

\end{exercisebox}

\begin{exercisebox}[Listing D.4: User Prompt (Distribution Sampling)]
Uniform Distribution: Generate one uniform random number in $[0,1)$.

Normal Distribution: Generate one standard normal random number $\mathcal{N}(0,1)$.

Exponential Distribution: Generate one exponential random number $Exp(1)$.

Beta Distribution: Generate one Beta random number $Beta(2, 5)$.

Gamma Distribution: Generate one Gamma random number $Gamma(2, 2)$.

\end{exercisebox}

\section{Prompt Used in Downstream Task}
\label{app:dt}

\begin{exercisebox}[Listing E.1: Medical Multiple-Choice Generation (User Prompt)]

Please generate a medical multiple-choice question (single answer, 4 options). The question should cover medical knowledge and be of moderate difficulty.

Requirements:

\quad (1) The question should have practical medical value.

\quad (2) All four options should be plausible with reasonable distractors.

\quad (3) Exactly one correct answer.

\quad (4) Cover different medical knowledge areas (e.g., internal medicine, surgery, pharmacology, pathology, diagnostics).

\quad (5) The correct answer should be UNIFORMLY distributed among A, B, C, D options across many independent generations (target: 25\% each), so as to avoid bias toward any particular option.

Output all required tags following the strict format given in the system prompt.

\end{exercisebox}

\begin{exercisebox}[Listing E.2: DSC System Prompt (Medical Multiple-Choice Generation)]

You are a medical education expert who creates high-quality multiple-choice questions.

1. Generate two independent 16-character random strings ($\text{seed}_1$, $\text{seed}_2$).

\qquad - Use printable ASCII characters. Make them look random and irregular.

\qquad - The two seeds must be completely independent.

2. Position-wise ASCII comparison to obtain a 16-bit string.

For each position $i$ ($0,\dots,15$):

\qquad - If $\text{seed}_1$[i] $>$ $\text{seed}_2$[i]: bit[i] = 1

\qquad - If $\text{seed}_1$[i] $\leq$ $\text{seed}_2$[i]: bit[i] = 0

3. Extract two bits at DISTANT positions to determine the correct-answer letter.

Take bit[0] (the 1st bit) and bit[8] (the 9th bit). These two positions are far apart to ensure independence.

Combine them into a 2-bit code and map to a letter:

\qquad - bit[0]=0, bit[8]=0  =>  $A$

\qquad - bit[0]=0, bit[8]=1  =>  $B$

\qquad - bit[0]=1, bit[8]=0  =>  $C$

\qquad - bit[0]=1, bit[8]=1  =>  $D$

4. Construct ONE medical multiple-choice question whose correct answer is placed at the letter determined in Step 3.

\qquad - The question should have practical medical value, with four plausible options and reasonable distractors.

\qquad - Cover different medical knowledge areas.

\qquad - Exactly one correct answer, placed at the letter sampled in Step 3.

You must perform all calculations yourself and show your work in each tag below.

Output format (strict, in this exact order):

\qquad <seed1>16 characters</seed1>

\qquad <seed2>16 characters</seed2>

\qquad <bits>16 binary digits</bits>

\qquad <bit0>0 or 1</bit0>

\qquad <bit8>0 or 1</bit8>

\qquad <correct answer> A or B or C or D </correct answer>

\qquad <question>question content</question>

\qquad <option a>option A content</option a>

\qquad <option b>option B content</option b>

\qquad <option c>option C content</option c>

\qquad <option d>option D content</option d>

\qquad <explanation>brief explanation</explanation>

\end{exercisebox}

\begin{exercisebox}[Listing E.3: Direct System Prompt (Medical Multiple-Choice Generation)]

You are a medical education expert who creates high-quality multiple-choice questions for medical students and professionals.

For each request, generate ONE medical multiple-choice question with four plausible options and decide which letter (A, B, C, or D) the correct answer should occupy. Do not use any fixed pattern; ensure the position of the correct answer is genuinely random across many independent generations.

Output format (strict, in this exact order):

\qquad <question>question content</question>

\qquad <option a>option A content</option a>

\qquad <option b>option B content</option b>

\qquad <option c>option C content</option c>

\qquad <option d>option D content</option d>

\qquad <correct answer> A or B or C or D </correct answer>

\qquad <explanation>brief explanation</explanation>

\end{exercisebox}

\begin{exercisebox}[Listing E.4: SSoT System Prompt (Medical Multiple-Choice Generation)]

You are a helpful AI Assistant designed to provide well-reasoned and detailed responses. If the task allows many possible answers, you must generate ONE diverse response for the task. For that, you must begin by generating a unique and complex random string to serve as a seed.

This random string should appear sufficiently complex and unpredictable, with no obvious structure or pattern. Use your judgment to ensure it looks arbitrary and unguessable.

If the user asks you some question which allows multiple answers, use the generated seed (the exact contents inside the <random string> tags) to guide any random sampling or stochastic decisions.

Follow these steps for every instruction:

\quad 1. Output the random seed string enclosed within <random string> and </random string> tags.

\quad 2. Think deeply and carefully about the user's question, and enclose this reasoning within <thinking> and </thinking> tags. You have to generate ONE response leveraging the generated seed---the exact contents inside the <random string> tags, to ensure your single answer is unique and diverse. Make sure to extract maximum randomness from the string by using all of its content.

\quad 3. Provide your final answer, enclosed within <answer> and </answer> tags.

Strictly follow this tag structure, and respond in the following format:

\qquad <random string>...</random string>

\qquad <thinking>...</thinking>

\qquad <answer>

\qquad \quad <question>question content</question>

\qquad \quad <option a>option A content</option a>

\qquad \quad <option b>option B content</option b>

\qquad \quad <option c>option C content</option c>

\qquad \quad <option d>option D content</option d>

\qquad \quad <correct answer> A or B or C or D </correct answer>

\qquad \quad <explanation>brief explanation</explanation>

\qquad </answer>

\end{exercisebox}

\begin{exercisebox}[Listing E.5: Attribute-Constrained Prompt Generation (User Prompt)]

Generate ONE text-to-image prompt describing a person wearing a coat.

You must INDEPENDENTLY sample four attributes according to the following target distributions:

1. Gender (Bernoulli):

   \qquad - Male:   49.49\%
   
   \qquad - Female: 50.51\%

2. Race/Ethnicity (Categorical, 5 classes):

   \qquad - White (Non-Hispanic):        57.46\%
   
   \qquad - Hispanic/Latino:             20.02\%
   
   \qquad - Black (Non-Hispanic):        12.63\%
   
   \qquad - Asian (Non-Hispanic):         6.49\%
   
   \qquad - Others (AIAN, NHPI, Mixed):   3.40\%

3. Height (Normal Distribution): $N(169.0, 10.0^2)$ cm

   \qquad - Approximately 68\% should be between 159–179 cm
   
   \qquad - Approximately 95\%  should be between 149–189 cm
   
   \qquad - Round the final height to an integer in cm.

4. Coat Color (Uniform Distribution):

   \qquad 7 colors with EQUAL probability (14.29\% each):
   
   \qquad Black, White, Red, Blue, Green, Yellow, Brown

Sample EACH attribute independently according to its target distribution, then produce ONE creative text-to-image prompt describing the sampled person wearing a coat. Output all required tags following the strict format given in the system prompt.

EXAMPLE: A white woman, 162cm tall, wearing a brown wool coat, walking through a sunlit autumn park.
\end{exercisebox}

\begin{exercisebox}[Listing E.6: DSC System Prompt (Attribute-Constrained Prompt Generation)]

You are a text-to-image prompt generator that samples attributes from specified statistical distributions.

1. Generate four independent pairs of 16-character random strings.

Produce 8 random 16-character strings using printable ASCII. Each pair ($\text{seed}_{1x}$, $\text{seed}_{2x}$) must be independent.

    \qquad - ($\text{seed}_{1a}$, $\text{seed}_{2a}$): used for Gender 
    
    \qquad - ($\text{seed}_{1b}$, $\text{seed}_{2b}$): used for Race/Ethnicity
    
    \qquad - ($\text{seed}_{1c}$, $\text{seed}_{2c}$): used for Height
    
    \qquad - ($\text{seed}_{1d}$, $\text{seed}_{2d}$): used for Coat Color
    
Please Remember:

    \qquad - Use printable ASCII characters. Make them look random and irregular.
    
    \qquad - The two seeds must be completely independent.

2. For each pair, perform position-wise ASCII comparison to obtain $16$ bits.

    \quad For each position $i$ ($0,\dots,15$), compare ASCII values and obtain a binary string bit[]:
    
    \qquad - If $\text{seed}_1$[i] $>$ $\text{seed}_2$[i]: bit[i] = 1
    
    \qquad - If $\text{seed}_1$[i] $\leq$ $\text{seed}_2$[i]: bit[i] = 0
    
    \quad You will give 4 independent 16-bit binary strings.
    
3. Convert each $16$-bit string to integer $N$, then compute $u = N / 65536$ (i.e. $N / 2^{16}$).

\quad You will obtain four independent uniform variates: $u_a, u_b, u_c, u_d$ in $[0, 1)$.
\end{exercisebox}

\begin{exercisebox}[Listing E.6 (continued): DSC System Prompt (Attribute-Constrained Prompt Generation)]

4. Map each $u$ to its attribute using the cumulative thresholds below (derived from the target distributions specified by the user).

\quad Gender (from $u_a$) - Bernoulli

\qquad - Male   \quad  if  $0.0000 \leq u_a < 0.4949$

\qquad - Female  \quad if  $0.4949 \leq u_a < 1.0000$

\quad Race/Ethnicity (from $u_b$) - Categorical

\qquad - White     \quad if  $0.0000 \leq u_b < 0.5746$

\qquad - Hispanic  \quad if  $0.5746 \leq u_b < 0.7748$

\qquad - Black     \quad if  $0.7748 \leq u_b < 0.9011$

\qquad - Asian     \quad if  $0.9011 \leq u_b < 0.9660$

\qquad - Others    \quad if  $0.9660 \leq u_b < 1.0000$

\quad Height (from $u_c$) - Normal $N(169.0, 10.0^2)$ cm

\quad Compute height = $169.0 + 10.0 \times \Phi^{(-1)}(u_c)$, where $\Phi^{-1}(u_c)$ is the inverse CDF of the standard normal distribution. Use your best knowledge of the standard normal quantile function. Round the final height to an integer in cm.

\quad Coat Color (from $u_d$) - Uniform over 7 colors (each bin has width $1/7 = 0.142857$)

\qquad - Black   \quad if  $0.000000 \leq u_d < 0.142857$

\qquad - White   \quad if  $0.142857 \leq u_d < 0.285714$

\qquad - Red     \quad if  $0.285714 \leq u_d < 0.428571$

\qquad - Blue    \quad if  $0.428571 \leq u_d < 0.571429$

\qquad - Green   \quad if  $0.571429 \leq u_d < 0.714286$

\qquad - Yellow  \quad if  $0.714286 \leq u_d < 0.857143$

\qquad - Brown   \quad if  $0.857143 \leq u_d < 1.000000$

5. Produce ONE creative text-to-image prompt describing the sampled person wearing a coat.

You must perform all calculations yourself and show your work in each tag below.

Output format (strict, in this exact order):

\qquad <seed1a>16 characters</seed1a>

\qquad <seed2a>16 characters</seed2a>

\qquad <seed1b>16 characters</seed1b>

\qquad <seed2b>16 characters</seed2b>

\qquad <seed1c>16 characters</seed1c>

\qquad <seed2c>16 characters</seed2c>

\qquad <seed1d>16 characters</seed1d>

\qquad <seed2d>16 characters</seed2d>

\qquad <gender>Male or Female</gender>

\qquad <race>White or Hispanic or Black or Asian or Others</race>

\qquad <height>integer cm</height>

\qquad <color>Black or White or Red or Blue or Green or Yellow or Brown</color>

\qquad <prompt>one creative text-to-image prompt</prompt>

\end{exercisebox}

\begin{exercisebox}[Listing E.7: Direct System Prompt (Attribute-Constrained Prompt Generation)]

You are a text-to-image prompt generator that must sample attributes from specified target distributions.

For each attribute, directly draw ONE sample from its target distribution using your own judgment. Do not use any fixed or predictable pattern. Make the sampled values look like genuine draws from the specified distributions.

Output format (strict, in this exact order):

\qquad <gender>Male or Female</gender>

\qquad <race>White or Hispanic or Black or Asian or Others</race>

\qquad <height>integer cm</height>

\qquad <color>Black or White or Red or Blue or Green or Yellow or Brown</color>

\qquad <prompt>one creative text-to-image prompt</prompt>

\end{exercisebox}

\begin{exercisebox}[Listing E.8: SSoT System Prompt (Attribute-Constrained Prompt Generation)]

You are a helpful AI Assistant designed to provide well-reasoned and detailed responses. If the task allows many possible answers, you must generate ONE diverse response for the task. For that, you must begin by generating a unique and complex random string to serve as a seed.

This random string should appear sufficiently complex and unpredictable, with no obvious structure or pattern. Use your judgment to ensure it looks arbitrary and unguessable.

If the user asks you some question which allows multiple answers, use the generated seed (the exact contents inside the <random string> tags) to guide any random sampling or stochastic decisions.

Follow these steps for every instruction:

\quad 1. Output the random seed string enclosed within  <random string> and </random string> tags.

\quad 2. Think deeply and carefully about the user's question, and enclose this reasoning within <thinking> and </thinking> tags. You have to generate ONE response leveraging the generated seed---the exact contents inside the <random string> tags, to ensure your single answer is unique and diverse. Make sure to extract maximum randomness from the string by using all of its content.

\quad 3. Provide your final answer, enclosed within <answer> and </answer> tags.

Strictly follow this tag structure, and respond in the following format:

\qquad <random string>...</random string>

\qquad <thinking>...</thinking>

\qquad <answer>

\qquad \quad <gender>Male or Female</gender>

\qquad \quad <race>White or Hispanic or Black or Asian or Others</race>

\qquad \quad <height>integer cm</height>

\qquad \quad <color>Black or White or Red or Blue or Green or Yellow or Brown</color>

\qquad \quad <prompt>one creative text-to-image prompt describing the sampled person wearing a coat</prompt>

\qquad </answer>

\end{exercisebox}

\section{Reproducibility and Raw Output Logging}
\label{app:reproducibility}

All experiments were executed through an OpenAI-compatible chat-completions endpoint. For each run, we recorded the requested model identifier, raw model output, parsed fields, and verification status. The requested model identifiers used in our main experiments were \texttt{claude-opus-4-6-thinking}, \texttt{gemini-3.1-pro-preview-thinking}, \texttt{MiniMax-M2.5}, \texttt{Qwen/Qwen3.5-27B}, and \texttt{Qwen/Qwen3.5-9B}. Unless otherwise specified, we used temperature \(=1.0\), \(n=200\) trials per repeat, and 5 repeats per model-distribution condition.

For each DSC trial, the stored record includes the raw response text, original seed pair, shuffled seed pair, comparison bits, integer \(N\), normalized variate \(u\), parsed final sample, parsing validity, and stage-wise verification results. For downstream tasks, we similarly store the raw response, parsed task attributes, and parsed output labels. These logs allow us to audit both distributional outcomes and intermediate arithmetic failures.


An example DSC raw-output excerpt is shown below; full raw responses are retained in the experiment logs.

\begin{exercisebox}[Listing F.1: DSC example output from \texttt{gemini-3.1-pro-preview-thinking} (Normal Distribution)]
\begin{verbatim}
(Explanation)
<original_seed1>k8#Lp2!Qx9v$mC5b</original_seed1>\n
<original_seed2>A3z*9dF&1qY@w7uH</original_seed2>\n
<seed1>p$kCx2L#9m!bv8Q5</seed1>\n
<seed2>F7@z&Hw*Aqu13Y9d</seed2>\n
<bits>1010100000011010</bits>\n
<N>43034</N>\n
<u>0.656646728515625</u>\n
<z>0.403328</z>
\end{verbatim}
\end{exercisebox}

For a downstream application, the stored raw output includes the sampled comparison bits and selected answer; full raw responses are retained in the experiment logs:

\begin{exercisebox}[Listing F.2: MCQ downstream example from \texttt{Qwen3.5-9B}]
\begin{verbatim}
<seed1>K*m@7#Lx9n^2qZ&wP</seed1>\n
<seed2>j$n&4%oV#z2rX(vM</seed2>\n
<bits>0101100110101101</bits>\n
<bit0>0</bit0>\n
<bit8>1</bit8>\n
<correct_answer>B</correct_answer>\n
<question>A 65-year-old man with known coronary artery disease presents to the 
emergency department with crushing substernal chest pain at rest lasting 45 minutes. 
ECG reveals ST-segment elevation in leads II, III, and aVF. Which of the following 
interventions should be administered FIRST after aspirin has been given?
</question>\n
<option_a>Vancomycin intravenous infusion</option_a>\n
<option_b>Unfractionated heparin intravenous drip</option_b>\n
<option_c>Oral metoprolol</option_c>\n
<option_d>Intravenous nitroglycerin</option_d>\n
<explanation>...</explanation>
\end{verbatim}
\end{exercisebox}

\begin{exercisebox}[Listing F.3: Attribute-constrained downstream example from \texttt{Qwen3.5-9B}]
\label{lst:attr-log}
\begin{verbatim}
(Explanation)
<seed1a>xK9#mP2l$vrQ8wZ</seed1a>\n
<seed2a>3L7!nN1o%uqW6xE@</seed2a>\n
<seed1b>A#4%M^7wZ9qR2xT</seed1b>\n
<seed2b>B@5&N^8xY0tS3yU</seed2b>\n
<seed1c>C&6*N^9yZ1rT4zU</seed1c>\n
<seed2c>D*7-O^0zA2sV5aW</seed2c>\n
<seed1d>E(8/P^1aB3tW6bX</seed1d>\n
<seed2d>F)9/Q^2bC4uX7cY</seed2d>\n
<gender>Female</gender>\n
<race>Black</race>\n
<height>175</height>\n
<color>Blue</color>\n
<prompt>A portrait of a beautiful Black woman wearing a stylish midnight blue coat
with subtle teal panels, standing in soft natural light.</prompt>
\end{verbatim}
\end{exercisebox}

\end{appendices}
\end{document}